%% file: iclr2027_conference.tex
\documentclass{article} %
\usepackage{iclr2027_conference,times}

\input{math_commands.tex}

\usepackage{hyperref}
\usepackage{url}
\usepackage{multirow}
\usepackage{booktabs}
\usepackage{graphicx}

\usepackage{amsmath}
\usepackage{amssymb}
\usepackage{graphicx}

\usepackage{multirow}
\usepackage{makecell}
\usepackage{caption}

\usepackage{adjustbox}
\usepackage{wrapfig}

\usepackage{float}
\usepackage{subfig}
\usepackage{algorithm}
\usepackage{algorithmic}
\usepackage[utf8]{inputenc} %
\usepackage[T1]{fontenc}    %
\usepackage{url}            %
\usepackage{booktabs}       %
\usepackage{amsfonts}       %
\usepackage{nicefrac}       %
\usepackage{microtype}      %
\usepackage{array}
\newcolumntype{C}[1]{>{\centering\arraybackslash}m{#1}}
\usepackage{xcolor}         %
\usepackage{color, colortbl}
\usepackage[normalem]{ulem}
\usepackage{framed}
\usepackage{makecell}
\usepackage{listings}
\usepackage{xspace}
\usepackage{pifont}

\usepackage{amsthm}

\theoremstyle{plain}
\newtheorem{theorem}{Theorem}[section]

\newtheorem{lemma}[theorem]{Lemma}

\theoremstyle{definition}

\newtheorem{assumption}[theorem]{Assumption}
\theoremstyle{remark}

\title{Progressive Agent Skill Generation via \\ Reinforcement Learning}

\author{
 Junhao Shen$^{1}$\quad Chengzhi Piao$^{1}$ \quad Zhanqiu Zhang$^{2\dag}$\quad Yiwen Guo$^{3\dag}$\quad Hong Cheng$^{1}$ \\
$^1$The Chinese University of Hong Kong \\
$^2$LIGHTSPEED \quad $^3$Independent Researcher\\
\texttt{shen.junhao@outlook.com} \quad \texttt{\{zqzhang27,guoyiwen89\}@gmail.com}\\
\texttt{hcheng@se.cuhk.edu.hk}
}

\newcommand\methodname{Skill-$\alpha$}

\iclrfinalcopy %
\begin{document}
\maketitle
\renewcommand\thefootnote{}
\footnotetext{$\dagger$ Corresponding authors.}

\input{sections/abstract}
\input{sections/introduction}

\input{sections/related_work}

\input{sections/preliminaries}
\input{sections/method}
\input{sections/experiments}
\input{sections/conclusion}

\bibliography{iclr2027_conference}
\bibliographystyle{iclr2027_conference}

\clearpage
\appendix
\setcounter{table}{0}
\setcounter{figure}{0}
\setcounter{equation}{0}
\renewcommand{\thetable}{A\arabic{table}}
\renewcommand\thefigure{A\arabic{figure}}
\renewcommand\theequation{A\arabic{equation}}
\input{sections/appendix}

\end{document}

%% file: math_commands.tex
\usepackage{amsmath,amsfonts,bm}

\providecommand{\eg}{\textit{e.g.}\@\xspace}
\providecommand{\ie}{\textit{i.e.}\@\xspace}

\def\eqref#1{equation~\ref{#1}}

\def\1{\bm{1}}

\DeclareMathAlphabet{\mathsfit}{\encodingdefault}{\sfdefault}{m}{sl}
\SetMathAlphabet{\mathsfit}{bold}{\encodingdefault}{\sfdefault}{bx}{n}

%% file: sections/abstract.tex
\begin{abstract}
Recent large language model agents often use external skills as modular procedural units that condition inference and improve complex task solving. Thus, automatically generating high-quality skills from documents or experience has become an important problem.
Existing skill generation methods largely rely on heuristics or pipeline-style consolidation, which must be specially designed for different evidence sources. In contrast, learning-based approaches offer a more unified way to model skill generation across heterogeneous sources. However, learning-based skill generation remains challenging because skills lack a natural supervision signal based on relevance or correctness; their value can largely be determined only by whether they improve the behavior of the agent on downstream tasks.
To address this challenge, we propose \methodname{}, a reinforcement learning method that learns a unified policy for progressive skill generation. 
Specifically, we construct each skill by repeatedly applying the learned policy to successive source evidence and introduce a novel rollback reward that evaluates each edit by comparing downstream execution under the original and edited skills on an anchored query.
Extensive experiments show that \methodname{} generates more effective skills than methods based on heuristics or pipelines in both document-to-skill and experience-to-skill settings. Under the main GPT-4o worker, \methodname{} improves average downstream success rates over the strongest skill-generation baseline by 3.1 points on CL-Bench and 6.7 points on tau2-bench. Further ablations and analysis validate the importance of rollback reward and progressive generation.
Code is available at \url{https://github.com/ejhshen/skill-alpha}.
\end{abstract}

%% file: sections/introduction.tex
\section{Introduction}

Recent large language model (LLM) agents~\citep{luo2025llmagent_survey,wang2024llm_autonomous_agent_survey} are usually expected to solve tasks that require multi-step reasoning~\citep{plaat2026multistep}, long-horizon planning~\citep{huang2024understanding}, and reliable tool use~\citep{shen2024llm}. A practical way to extend these agents is to provide external skills that condition inference and guide how the agent decomposes tasks, invokes tools, and checks intermediate results~\citep{wang2025reinforcement,wang2023voyager,shen2026dynamic}. As skills become reusable modules for shaping agent behavior without retraining the model, automatically generating high-quality skills from documents or experience becomes an important problem~\citep{anthropic_skills}.

Existing skill generation methods largely rely on heuristics, prompting, or pipeline-style consolidation. Document-to-skill methods compress documents, rules, or task descriptions into procedural instructions~\citep{si2026context,yang2026autoskill}, while experience-to-skill methods distill successful or failed executions into reusable guidance~\citep{ni2026trace2skill,mi2026skill}. Although effective in specific settings, these methods usually require separate designs for different evidence sources and provide limited guidance on how each piece of evidence should change the skill being written. Thus, the central challenge is not only to produce skill text, but to learn a unified skill generation process.

However, learning to generate skills is difficult because skills do not have a natural supervision signal based on relevance or correctness. Unlike mathematical reasoning or web search, where outputs can often be evaluated against explicit factual answers or execution results~\citep{guo2025deepseek,jin2025search}, the quality of a skill lacks a similarly direct supervision signal and is often assessed through its effect on downstream task performance. A fluent skill may still be redundant, over-specific, or misleading, while a compact edit may substantially improve the worker agent's behavior~\citep{huang2026rawexperienceskillconsumption,zhou2026skillgenbench}. This makes execution-grounded credit assignment the key obstacle to learning-based skill generation.

To address this issue, we propose \methodname{}, a reinforcement learning framework that learns a unified policy for progressive skill generation. The learned policy is repeatedly applied over source evidence, with each edited skill serving as the state for the next construction step.
Specifically, given a skill draft and new evidence, the generator learns whether to add new procedures, revise imprecise rules, merge redundant guidance, remove harmful or over-specific content, or leave the skill unchanged. We introduce rollback reward as the key training signal. After a candidate edit is applied, the original and edited skills are evaluated on the same evidence-related anchored query and compared by a task-level verifier, so that the reward reflects the local effect of changing the skill condition.
In this way, \methodname{} turns skill generation pipelines into a unified learning problem over local skill edits, enabling the skill generator to be trained with reinforcement learning.

We evaluate \methodname{} in both document- and experience-to-skill settings. Under the GPT-4o worker, \methodname{} improves average downstream success rates over the strongest baselines by 3.1 points on CL-Bench~\citep{dou2026cl} and 6.7 points on tau2-bench~\citep{barres2025tau2}. On the fully unseen BFCL multi-turn benchmark~\citep{patil2023gorilla}, \methodname{} is the only skill-generation method that outperforms the no-skill worker under different backbones. Ablations further validate the importance of rollback reward and progressive generation.

Overall, our contributions are threefold. First, we formulate agent skill generation as a progressive decision-making problem, unifying document-to-skill and experience-to-skill generation under a single learning framework. Second, we introduce rollback reward, which provides execution-grounded credit assignment for local skill edits. Third, we empirically show that \methodname{} produces effective skills from heterogeneous evidence sources, with ablations and analyses validating the roles of rollback reward and progressive generation.

%% file: sections/related_work.tex
\section{Related Work}
\noindent\textbf{Agent Skills.}
Agent skills are reusable procedural modules that condition the inference process of an agent and reshape its behavior on downstream tasks~\citep{anthropic_skills}. Recent work has studied skills from multiple perspectives, including skill evolution~\citep{shen2026dynamic,skillrl2026,ouyang2026skillos}, selection~\citep{zeng2026group,liu2026graph,zheng2026skillrouter}, utilization~\citep{wang2023voyager,wang2025reinforcement}, and internalization~\citep{wang2026skill}. These studies show that skills are functional units that can guide planning, tool use, and long-horizon task solving~\citep{wang2026webxskill,wang2026effiskill,xu2025speci}. Because of this role, more attention has shifted toward skill generation and acquisition~\citep{si2026context,yang2026autoskill,ni2026trace2skill,mi2026skill} and benchmark evaluation~\citep{huang2026rawexperienceskillconsumption,zhou2026skillgenbench,skillsbench2026}. \methodname{} follows this direction by treating automatic external-skill generation as the target of learning.

\noindent\textbf{Skill Generation and Acquisition.}
Skill generation and acquisition focus on producing reusable skill modules from evidence sources such as documents, task contexts, experience, and trajectories. Early work explored converting verifier feedback into reusable procedural memory, guidelines or workflows, including ExpeL~\citep{zhao2024expel}, Agent Workflow Memory~\citep{wang2025awm}, AutoGuide~\citep{fu2024autoguide} and AutoFlow~\citep{li2024autoflow}. Related work has also studied automatic skill construction through prompt-based skill creation~\citep{anthropic_skill_creator}. More recent pipeline-based skill generation systems generally fall into two categories. For document-to-skill settings, Ctx2Skill~\citep{si2026context} and AutoSkill~\citep{yang2026autoskill} aim to compress task documents, rules, or contexts into executable procedural instructions; for experience-to-skill settings, Trace2Skill~\citep{ni2026trace2skill}, SkillX~\citep{skillx2026}, and SkillPro~\citep{mi2026skill} distill trajectories, feedback, or hierarchical experience into reusable skill knowledge. These methods automate skill acquisition but rely on heuristic or pipeline-style consolidation tailored to specific evidence types. 
In contrast, \methodname{} learns one unified editing policy and repeatedly applies it to construct skills from heterogeneous evidence sources.

\noindent\textbf{Reinforcement Learning for External Module Optimization.}
Reinforcement learning (RL) has become a common approach for optimizing LLM agents to enhance slow-thinking reasoning capabilities~\citep{schulman2017proximal,guo2025deepseek,yu2024dapo,shen2026sophia,agentr12025}. Beyond optimizing the acting policy itself, recent work also applies RL to external module management. In tool-use and web-search settings, RL is used to improve how agents invoke external tools~\citep{li2025torl,singh2025agentic}, interact with environments~\citep{agentr12025,shen2026achieving}, and use search results~\citep{qi2024webrl,jin2025search} to complete complex tasks. Memory-R1~\citep{yan2025memory} and Mem-$\alpha$~\citep{wang2025mem} learn memory operations through downstream question-answering rewards. In contrast, \methodname{} generates transferable skills from documents and experience. Thus, we evaluate local edits by comparing worker execution under the original and edited skills, directly supervising changes. Recent works like Skill-R1~\citep{vishe2026skillr1} recurrently optimize a task-specific skill using verified rollout histories for each task instance, whereas \methodname{} learns a unified evidence-to-skill generation policy that is applied to new tasks without task-specific parameter optimization.

%% file: sections/preliminaries.tex
\section{Preliminaries}
\label{sec:preliminaries}

\noindent\textbf{LLM Agent.}
We model a worker agent as a fixed policy $\pi_{\psi}$ that interacts with an environment to solve a task. Given a task instruction or context $q$, the worker produces an execution trajectory $\tau=(o_1,a_1,\ldots,o_H,a_H)$, where $H$ is the trajectory length and $o_h$ and $a_h$ denote the observation and action at worker step $h$. An external skill $z$ is inserted into the worker context. The skill-conditioned worker induces the trajectory distribution $\pi_{\psi}(\tau \mid z,q)$. In this work, $\pi_{\psi}$ is not the training target. We instead train a skill-editing policy $\pi_\phi$ that samples a local edit action $A_t$ conditioned on the current skill state and source evidence at each step. Applying these actions progressively produces the skill $z$ that conditions the fixed worker on downstream tasks.

\noindent\textbf{Group Relative Policy Optimization.}
Group Relative Policy Optimization (GRPO)~\citep{guo2025deepseek} is used to optimize the skill generator. For an input $X$, GRPO samples a group of edit actions $\{A_i\}_{i=1}^{G}$ from the old policy $\pi_{\phi_{\mathrm{old}}}$ and assigns each action a scalar reward $r_i$. The group-relative advantage is $\widehat{\mathsf{A}}_i = \frac{r_i-\mathrm{mean}(\{r_1,\ldots,r_G\})}{\mathrm{std}(\{r_1,\ldots,r_G\})}$. The policy is then optimized with the objective
\begin{equation}
\begin{aligned}
\mathcal{J}_{\mathrm{GRPO}}(\phi)
=
\mathbb{E}
\left[
\frac{1}{G}\sum_{i=1}^{G}
\left(
\min\left(
\rho_i(\phi)\widehat{\mathsf{A}}_i,
\mathrm{clip}(\rho_i(\phi),1-\epsilon,1+\epsilon)\widehat{\mathsf{A}}_i
\right)
-\beta D_{\mathrm{KL}}\!\left(\pi_{\phi}\|\pi_{\mathrm{ref}}\right)
\right)
\right],
\end{aligned}
\label{eq:grpo}
\end{equation}
where $\rho_i(\phi)=\frac{\pi_{\phi}(A_i\mid X)}{\pi_{\phi_{\mathrm{old}}}(A_i\mid X)}$, $\pi_{\mathrm{ref}}$ is the reference policy, and $\epsilon,\beta$ are hyperparameters.

\noindent\textbf{Skill Generation Formalization.}
For each task family $\mathcal{M}$, let $p^\star_{\mathcal{M}}$ denote a conceptual query-conditioned joint distribution over evidence and high-quality teacher behavior. Let $x_{1:T}$ denote a sequence of $T$ evidence units (\eg, document segments or batches of execution traces). For a source query $q'$, we sample $x_{1:T}\sim p^\star_{\mathcal{M}}(\cdot\mid q')$ from its evidence marginal, while $p^\star_{\mathcal{M}}(\tau\mid q')$ denotes the corresponding teacher behavior marginal. Given an initial skill state $z_0$, the skill-editing policy $\pi_\phi$ samples local edit actions whose progressive application induces a distribution $\bar{\pi}_\phi(\cdot\mid x_{1:T},z_0)$ over the final skill $z$. The value of a generated skill is defined by how it changes the behavior distribution of the fixed worker agent $\pi_{\psi}(\tau\mid z,q)$ on a held-out target query $q$ from the same task family, where typically $q'\neq q$. The ideal objective of skill generation is
\begin{equation}
\min_{\phi}
\mathbb{E}
\left[
D_{\mathrm{KL}}\!\left(
p^\star_{\mathcal{M}}(\tau\mid q)
\;\|\;
\pi_{\psi}(\tau\mid z,q)
\right)
\right],
\label{eq:ideal_skill_generation}
\end{equation}
where the expectation follows the source-target sampling process above over $\mathcal{M}$, $q'$, $x_{1:T}$, $q$, and $z\sim\bar{\pi}_{\phi}(\cdot\mid x_{1:T},z_0)$. This objective asks the generated skill to move the fixed worker toward high-quality behavior on held-out queries. Since $p^\star_{\mathcal{M}}(\tau\mid q)$ is unavailable, we optimize the execution-grounded rollback signal introduced in Section~\ref{sec:rollback_grounded_editing}.

%% file: sections/method.tex
\begin{figure}[t]
    \centering
    \includegraphics[width=0.85\linewidth]{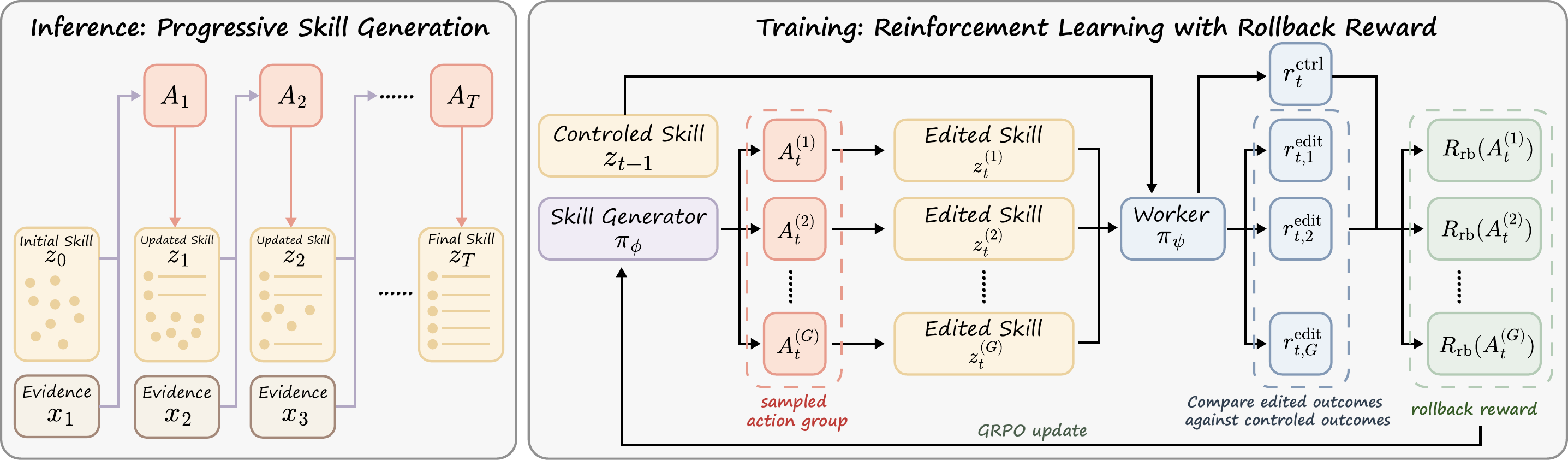}
    \caption{An overview of \methodname{}. Left: during inference, \methodname{} reads evidence sequentially and applies a sequence of local edit actions to progressively generate the skill from the initial state $z_0$ to the final skill $z_T$. Right: during training, the skill generator samples a group of candidate actions from the current skill state, constructs the edited skills, and evaluates them on the same anchored query; the resulting rollback rewards are then used for the GRPO update.}
    \vspace{-13pt}
    \label{fig:overview}
\end{figure}

\section{Method: \methodname{}}
\label{sec:method}
To approximate the ideal objective in Eq.~\ref{eq:ideal_skill_generation}, we develop \methodname{} as a trainable framework for progressive skill generation. Specifically, we first introduce how to decompose skill generation into local edits in Section~\ref{sec:progressive_skill_generation}, then show the rollback reward for edit-level credit assignment in Section~\ref{sec:rollback_grounded_editing}, and present the training implementation in Section~\ref{sec:training_implementation}. We illustrate the overall framework in Figure~\ref{fig:overview}.

\subsection{Progressive Skill Generation}
\label{sec:progressive_skill_generation}
Ideally, a generator should read all source evidence $x_{1:T}$ induced by $q'$ and produce a complete skill in one shot, regardless of the generation strategy. In practice, the source evidence often exceeds the context window of an LLM because individual units are too long or the sequence contains too many units, causing key information to be missed. Moreover, one-shot skill generation entangles several operations, including abstraction, denoising, conflict resolution, and compression, which makes credit assignment weak and training difficult.
To address this issue, we construct the skill through a sequence of local edit decisions rather than one-shot generation.
Starting from an initial skill $z_0$, the generator reads evidence sequentially and updates the current skill one step at a time, \ie,
\begin{equation}
A_t\sim\pi_\phi(\cdot\mid z_{t-1},x_t),
\qquad
z_t=\mathrm{Edit}(z_{t-1},A_t).
\label{eq:progressive_skill_edit}
\end{equation}
Here, $A_t$ is an edit conditioned on the current skill $z_{t-1}$ and evidence $x_t$, and $\mathrm{Edit}(\cdot)$ applies it to produce $z_t$, which the same policy uses to process $x_{t+1}$. After $T$ steps, the final skill $z_T$ is a sample from the induced distribution $\bar{\pi}_{\phi}(\cdot\mid x_{1:T},z_0)$ in Eq.~\ref{eq:ideal_skill_generation}.
This progressive formulation covers both document- and experience-to-skill settings, reducing generation to a unified local decision problem, \ie, how should each evidence unit be used to edit the current skill so that the fixed worker $\pi_\psi$ performs better on held-out queries? We use a rollback reward to supervise these edit decisions.

\subsection{Action Space and Rollback Reward}
\label{sec:rollback_grounded_editing}
We define the edit action space $\mathcal{A}=\{\textsc{Create}, \textsc{Update}, \textsc{Merge}, \textsc{Prune}, \textsc{Noop}\}$. Each sampled action edits a single skill artifact. \textsc{Create} adds missing rules or procedures; \textsc{Update} fixes incomplete or inaccurate rules; \textsc{Merge} consolidates overlapping content; \textsc{Prune} removes misleading, overly specific, or redundant content; and \textsc{Noop} leaves the skill unchanged when the current evidence provides no useful new information.

The remaining challenge is edit-level reward assignment. Directly rewarding an edit by rolling out the worker agent with the edited skill is insufficient, since a correct downstream answer may come from the capability of the worker agent or from an easy query, rather than from the edit itself. Thus, we introduce rollback reward, which assigns credit by comparing the old and edited skills on the same anchored query.
For an edit decision at step $t$, we attach an anchored query $q_t^{\mathrm{anc}}$ selected from queries associated with the same evidence source, and a task-level verifier $\mathcal{V}_t$. After applying $A_t$ to obtain $z_t$, the fixed worker answers the same anchored query under the original and edited skills:
\begin{align}
a_t^{\mathrm{ctrl}} \sim \pi_\psi(\cdot\mid q_t^{\mathrm{anc}},z_{t-1}), \quad
a_t^{\mathrm{edit}} \sim \pi_\psi(\cdot\mid q_t^{\mathrm{anc}},z_t)\,.
\end{align}
The verifier then assigns scalar feedback to the two outcomes, \ie, $r_t^{\mathrm{ctrl}}=\mathcal{V}_t(a_t^{\mathrm{ctrl}})$ and $r_t^{\mathrm{edit}}=\mathcal{V}_t(a_t^{\mathrm{edit}})$. In our main setup, CL-Bench uses official task-specific rubrics, while SpreadsheetBench and tau2-bench use direct benchmark-side environment feedback on the same anchored evaluation query. In GRPO training, rollback rewards are assigned within each sampled action group. For a sampled candidate $A_t^{(i)}$, we write its edited verifier output as $r_{t,i}^{\mathrm{edit}}$. Formally,

\begin{equation}
R_{\mathrm{rb}}(A_t^{(i)})=
\begin{cases}
1, & \text{if } A_t^{(i)}\neq\textsc{Noop} \text{ and } r_{t,i}^{\mathrm{edit}} > r_t^{\mathrm{ctrl}},\\
1, & \text{if } A_t^{(i)}=\textsc{Noop} \text{ and no other valid sampled edit achieves } \\
   & \text{a verifier output larger than } r_t^{\mathrm{ctrl}},\\
0, & \text{otherwise}.
\end{cases}
\label{eq:rollback_reward}
\end{equation}
This local reward provides edit-level credit assignment during RL, while full-task execution remains the evaluation target. Appendix~\ref{app:theory} shows that, for non-\textsc{Noop} edits, expected rollback reward preserves ideal-preference rankings under a calibrated verifier and gives an exact success-probability ordering for a binary verifier, providing a local justification for group-relative optimization.

\subsection{Training Implementation}
\label{sec:training_implementation}
\begin{table}[t]
\centering
\small
\begin{tabular}{p{0.8\linewidth}}
\toprule[0.5pt]
\textbf{system:} Skill generation action policy.\\[2pt]
\textbf{user:} \\
\quad \texttt{\#\# Current SKILL.md}\\
\quad \texttt{<skill>} current skill content \texttt{</skill>}\\
\quad \texttt{\#\# Evidence}\\
\quad \texttt{<evidence>} serialized source evidence \texttt{</evidence>}\\[2pt]
\textbf{assistant:} \\
\quad \texttt{<think>} evidence-grounded diagnosis and edit rationale \texttt{</think>}\\
\quad \texttt{<action>} \{"action": "...", ...\} \texttt{</action>}\\
\bottomrule[0.5pt]
\end{tabular}
\caption{Training template for \methodname{}. The policy writes a reasoning trace and one structured edit action for the current skill and evidence. Detailed action format is shown in Table~\ref{tab:appendix_skill_alpha_actions}.}
\vspace{-17pt}
\label{tab:skill_alpha_template}
\end{table}

We train $\pi_\phi$ as a structured skill-editing policy by initializing it from instruction-tuned Qwen3-8B~\citep{yang2025qwen3}, warming it up with supervised skill-editing data, and then optimizing it with GRPO. Each sample contains the current skill $z_{t-1}$ and an evidence batch $x_t$, together with an anchored query $q_t^{\mathrm{anc}}$ and verifier $\mathcal{V}_t$. The policy observes only $(z_{t-1},x_t)$, while the anchored query and verifier are used exclusively to compute rollback reward.
As shown in Table~\ref{tab:skill_alpha_template}, the policy outputs a reasoning trace and one structured edit action for the current state. The input construction is shared across both evidence types. For document-to-skill, short contexts are used directly, while long contexts are decomposed into ordered natural segments and processed progressively. For experience-to-skill, evidence is built from trajectories sampled on the same training split used for \methodname{}, then serialized into a compact view that preserves the task identity, key steps, and environment feedback. In the main setting, each experience batch contains up to $4$ traces, each trace contributes at most $8$ steps and $3{,}000$ characters, and the prompt budget is capped at $8{,}192$ tokens. This shared serialization keeps the local editing state within context limits while exposing evidence structure across benchmarks.

During the SFT stage, the warm-up data is synthesized by DeepSeek-V4-Pro~\citep{xu2026deepseek} from the training splits of CL-Bench, SpreadsheetBench, and tau2-bench. We keep only filtered edit trajectories that follow the shared \texttt{<think>}+\texttt{<action>} format, so that the policy learns the action syntax, local edit structure, and basic evidence-grounded editing behavior under the same interface used later in RL. The resulting checkpoint is then used to initialize GRPO.

After warm-up, we optimize the policy with GRPO while keeping the worker agent fixed as GPT-4o~\citep{gpt4o}. 
For a local editing state $(z_{t-1},x_t)$, the old policy samples a group of candidate actions $A_t^{(1)},\ldots,A_t^{(G)} \sim \pi_{\phi_{\mathrm{old}}}(\cdot \mid z_{t-1},x_t)$. 
For CL-Bench, $q_t^{\mathrm{anc}}$ is a training query associated with the same source context as the local editing state, and $\mathcal{V}_t$ is a GPT-5.5~\citep{gpt55} rubric judge instantiated with the official task-specific evaluation rules. 
For SpreadsheetBench and tau2-bench, $q_t^{\mathrm{anc}}$ is sampled from unsuccessful or partially successful training trajectories in the same source group and excluded from the evidence of that state. $\mathcal{V}_t$ is the direct benchmark-side feedback returned by the same anchored execution interface.
The current skill is evaluated once on the anchored query to obtain the control score $r_t^{\mathrm{ctrl}}$, and each edited skill is then evaluated with the same worker and verifier to obtain $r_{t,1}^{\mathrm{edit}},\ldots,r_{t,G}^{\mathrm{edit}}$. The resulting scores are converted into rollback rewards by Eq.~\ref{eq:rollback_reward}. 
Thus, GRPO optimizes the unified editing policy using rollback rewards at individual edit states. Repeated application of this policy produces the progressive skill-construction process in Eq.~\ref{eq:progressive_skill_edit}.

\input{algo/skill_alpha}

For training, SFT uses $8$ GPUs, a training batch size of $32$, a micro-batch size of $1$ per GPU, a maximum sequence length of $32{,}768$, a learning rate of $5\times 10^{-6}$, a weight decay of $0.01$, a warmup ratio of $0.05$, cosine scheduling, and $3$ epochs. GRPO uses a group size of $G=8$, a training batch size of $8$, an actor learning rate of $5\times 10^{-7}$, a rollout temperature of $1.0$, an entropy coefficient of $0.001$, a maximum prompt length of $8{,}192$, a maximum response length of $32{,}768$, and a maximum model length of $40{,}960$, with KL regularization disabled. The overall procedure is shown in Algorithm~\ref{alg:skill_alpha}. More detailed implementation details, including exact evidence preprocessing, anchored-query construction, and verifier protocols, are given in Appendix~\ref{app:skillalpha_impl}.

%% file: algo/skill_alpha.tex
\begin{algorithm}[t]
\caption{Reinforcement Learning with Rollback Reward for \methodname{}}
\begin{algorithmic}[1]
\STATE \textbf{Inputs:} edit-state dataset $\mathcal{D}$, skill generator $\pi_\phi$, fixed worker $\pi_\psi$, and group size $G$.
\FOR{each $(z_{t-1},x_t,q_t^{\mathrm{anc}},\mathcal{V}_t)\in\mathcal{D}$}
    \STATE Evaluate the control skill $z_{t-1}$ on the anchored query $q_t^{\mathrm{anc}}$ and obtain $r_t^{\mathrm{ctrl}}$.
    \STATE Sample $G$ candidate edit actions $\{A_t^{(i)}\}_{i=1}^{G}$ from $\pi_{\phi_{\mathrm{old}}}(\cdot\mid z_{t-1},x_t)$.
    \FOR{$i=1,2,\ldots,G$}
        \STATE Apply $A_t^{(i)}$ to the current skill and obtain $z_t^{(i)}$.
        \STATE Evaluate the edited skill $z_t^{(i)}$ on the same anchored query $q_t^{\mathrm{anc}}$ and obtain $r_{t,i}^{\mathrm{edit}}$.
    \ENDFOR
    \STATE Compute the rollback rewards $\{R_{\mathrm{rb}}(A_t^{(i)})\}_{i=1}^{G}$ with Eq.~\ref{eq:rollback_reward}, including the \textsc{Noop} fallback rule.
    \STATE Update $\pi_\phi$ on the candidate group using GRPO as defined in Eq.~\ref{eq:grpo}.
\ENDFOR
\STATE \textbf{Return:} Optimized skill generator $\pi_\phi$.
\end{algorithmic}
\label{alg:skill_alpha}
\end{algorithm}

%% file: sections/experiments.tex
\section{Experiments}
\providecommand{\best}[1]{\textbf{#1}}
\providecommand{\second}[1]{\underline{#1}}

\subsection{Evaluation Setup}
\noindent\textbf{Benchmarks and Baselines.}
We evaluate \methodname{} in both document- and experience-to-skill settings. For document-to-skill, we use CL-Bench~\citep{dou2026cl}. Training uses Rule System Application and Procedural Task Execution as source categories, and evaluation reports held-out tasks from these two categories together with the unseen categories Domain Knowledge Reasoning and Empirical Discovery \& Simulation. Since CL-Bench is evaluated with LLM-as-judge under task-specific rubrics, we use GPT-5.5 as the CL-Bench judge for all reported results. We compare against \textsc{No Skill}, Anthropic Skill-Creator~\citep{anthropic_skill_creator}, Progressive Prompt Skill, Ctx2Skill~\citep{si2026context}, and AutoSkill~\citep{yang2026autoskill}.
For experience-to-skill, we use SpreadsheetBench~\citep{ma2024spreadsheetbench} and tau2-bench~\citep{barres2025tau2}. SpreadsheetBench tests execution-based spreadsheet manipulation, while tau2-bench tests workflow skill reuse across Airline, Retail, and Telecom. 
To evaluate transfer beyond the data used for training, we additionally use BFCL multi-turn~\citep{patil2023gorilla} as an unseen experience-to-skill benchmark. No BFCL trajectory is used to train \methodname{}. The experience baselines are \textsc{No Skill}, Anthropic Skill-Creator, Progressive Prompt Skill, ExpeL~\citep{zhao2024expel}, Agent Workflow Memory (AWM)~\citep{wang2025awm}, Trace2Skill~\citep{ni2026trace2skill}, SkillX~\citep{skillx2026}, and SkillPro~\citep{mi2026skill}. Baseline details are provided in Appendix~\ref{app:baseline_impl}.

\noindent\textbf{Evaluation.}
All methods follow the same source-evidence and held-out-task protocol. For document-to-skill, every method receives the same source contexts and is evaluated on the same held-out tasks after injecting the generated skill into the worker. For experience-to-skill, the source trajectories are sampled once by GPT-4o on the benchmark training split, and this source split is exactly the trajectory training split used by \methodname{}; all automatic baselines receive the same trajectories, tool observations, and execution feedback for skill construction. We report same-worker evaluation with GPT-4o and cross-worker transfer to Claude-Sonnet-4.5~\citep{cluade45sonnet}. Note that Claude-Sonnet-4.5 is used only as a downstream worker that consumes the same skills generated from the GPT-4o trajectory pool. We report pass rates on CL-Bench, SpreadsheetBench, tau2-bench, and BFCL multi-turn. Detailed benchmark splits, worker settings, and leakage-prevention rules are given in Appendix~\ref{app:setup}.

\subsection{Main Results}
\input{tables/context_main_results}

\noindent\textbf{Document-to-skill.}
Table~\ref{tab:context_main_results} shows that \methodname{} achieves stable skill-construction performance on CL-Bench. Under GPT-4o, the largest gain appears on Procedural Task Execution, where \methodname{} improves from $4.30$ without skills to $9.68$. This suggests that \methodname{} is turning documents into executable skills instead of merely shortening long contexts. On other splits, \methodname{} is best or near-best, indicating that the generated skills preserve information useful for both procedural and reasoning-heavy tasks.
Under Claude-Sonnet-4.5, \methodname{} is best or tied-best on all four categories and in the average. Since this backbone already has a relatively strong \textsc{No Skill} baseline, these gains indicate that the generated skills not only fit the backbone used for skill construction but also provide reusable knowledge that transfers to another strong worker.

\input{tables/experience_main_results}

\noindent\textbf{Experience-to-skill.}
Table~\ref{tab:experience_main_results} shows that the advantage of \methodname{} is even clearer in experience-to-skill. Under GPT-4o, it is best or tied-best on every reported metric, improving performance on SpreadsheetBench and achieving a tau2-bench average of $55.83$.
These gains suggest that \methodname{} learns reusable execution strategies rather than merely recording experience or writing static rules. Under Claude-Sonnet-4.5, \methodname{} remains best or tied-best on SpreadsheetBench, Airline, Telecom, and the tau2-bench average, and stays close on Retail, showing that the learned skills transfer across workers even though Claude only consumes GPT-4o-generated skills.

\input{tables/bfcl_unseen_domain_results}

\noindent\textbf{Generalization to Unseen Domains.}
Table~\ref{tab:bfcl_unseen_domain_results} evaluates transfer to BFCL, which is fully excluded from \methodname{} training. At test time, each method uses only BFCL source trajectories to construct Sub-API skills and is evaluated on disjoint held-out tasks, measuring out-of-distribution evidence-to-skill generalization rather than task memorization. \methodname{} achieves the highest average under both workers and is the only skill-generation method that exceeds \textsc{No Skill} in both settings. While several baselines incur negative transfer on the unseen domains, \methodname{} achieves average gains of $1.01$ and $4.02$ points over \textsc{No Skill} under GPT-4o and Claude-Sonnet-4.5, respectively. These results support the transfer of the learned policy to unseen evidence domains across worker backbones.

\noindent\textbf{Effectiveness of Training.}
To examine the effectiveness of training, we compare against two strong prompt baselines, \ie, Anthropic Skill-Creator and Progressive Prompt Skill. They can be competitive on individual columns under both backbones, but their gains are unstable across tasks and workers. Anthropic drops to $7.50$ on Telecom, below the \textsc{No Skill} baseline of $12.50$, and Progressive drops to $13.50$ on SpreadsheetBench, far below the \textsc{No Skill} baseline of $26.00$. Since Progressive already performs iterative refinement, the more consistent gains of \methodname{} support the effectiveness of learning a skill-editing policy over repeatedly applying a fixed prompt.

\noindent\textbf{Generalization Across Evidence Sources and Workers.}
Taken together, the main results support the central claim of this paper. Most baselines are evidence-source specific. AutoSkill and Ctx2Skill are designed for document evidence, whereas Trace2Skill, ExpeL, AWM, SkillX, and SkillPro are designed for experience evidence. In contrast, \methodname{} uses one progressive skill editing framework across both sources and remains strong in both settings.
The cross-worker results provide a second form of generalization evidence. Gains are larger under GPT-4o, suggesting that the generated skills effectively supplement the backbone used for skill construction. At the same time, the same skills remain useful under Claude-Sonnet-4.5, which indicates that \methodname{} is learning externally usable task-solving knowledge rather than a worker-specific prompt shortcut.

\begin{figure}[t]
    \centering
    \vspace{0.8em}
    \includegraphics[width=0.80\linewidth]{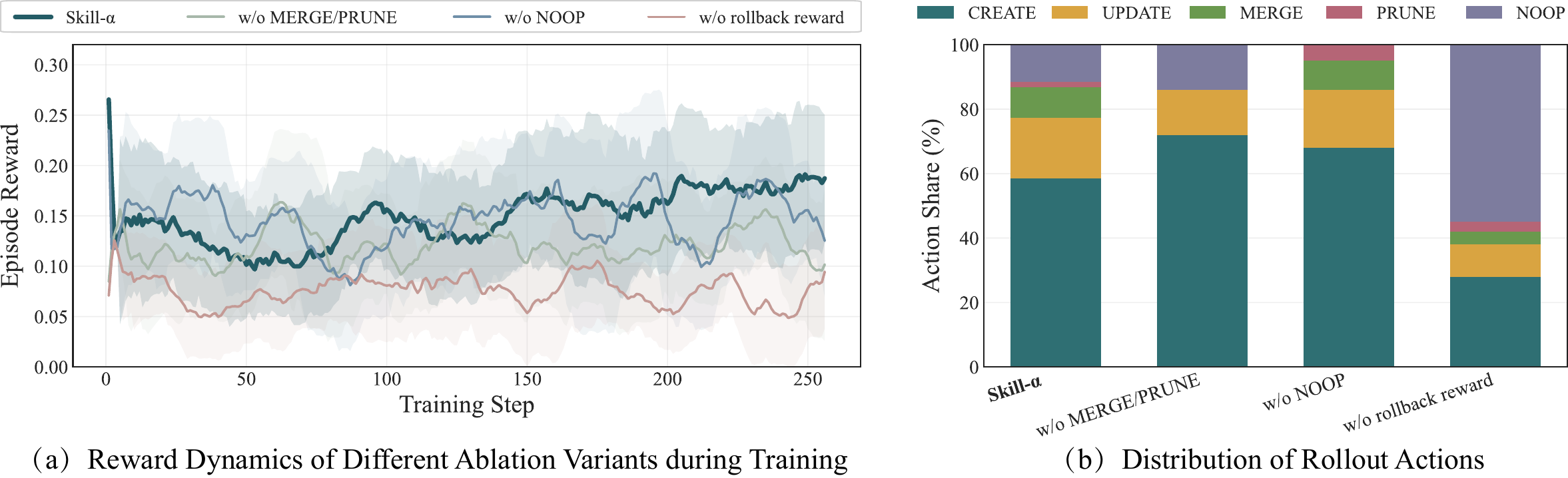}
    \caption{Training dynamics and rollout action distribution under the main ablations. Left: training reward over RL steps for \methodname{} and its RL ablations. Right: action distribution accumulated over rollout steps. ``SFT only'' is omitted because it has no RL rollout trajectory.}
    \vspace{-16pt}
    \label{fig:training_dynamics}
\end{figure}

\subsection{Ablation Study and Training Dynamics}
\noindent
We study four ablations under the same evaluation protocol. ``SFT only'' removes RL and keeps only the supervised warm start. ``w/o rollback reward'' replaces rollback reward with direct verifier reward. ``w/o \textsc{Merge}/\textsc{Prune}'' removes the structural consolidation and deletion actions from the edit space. ``w/o \textsc{Noop}'' removes the abstention action and forces every step to modify the current skill.

\noindent\textbf{Ablation Performance.}
Table~\ref{tab:ablation_results} shows that the full model is strongest on all three benchmarks, which confirms that the benefit does not come from learning the output format alone. Without rollback reward, performance stays close to ``SFT only'', which makes rollback reward the key signal tying a local edit to downstream improvement. Removing \textsc{Merge}/\textsc{Prune} also causes a clear drop, especially on the tau2-bench average, showing that skill generation requires explicit consolidation and deletion rather than continual accumulation alone. Removing \textsc{Noop} still leaves a relatively strong model, which rules out a trivial no-edit explanation, but the full model remains consistently better, indicating that \textsc{Noop} is useful as a group-level fallback when no sampled edit improves over the control.

\begin{wraptable}{r}{0.58\columnwidth}
\centering
\vspace{-0.8em}
\caption{Ablation study. CL-Bench Avg. averages CL-Bench categories, and tau2-bench Avg. averages Airline, Retail, and Telecom. All numbers are pass rates (\%).}
\label{tab:ablation_results}
\input{tables/ablation_results}
\vspace{-0.8em}
\end{wraptable}
\noindent\textbf{Training Dynamics \& Edit Behavior.}
Figure~\ref{fig:training_dynamics} explains the ablation results from a complementary perspective. 
After an initial drop, the reward of the full model gradually recovers and reaches the strongest performance in later training. Together with its diverse action distribution, this trend suggests that the model learns an effective editing policy rather than relying on a single high-reward action.
Specifically, \textsc{Create} is the dominant action, but \textsc{Update}, \textsc{Merge}, \textsc{Prune}, and \textsc{Noop} all remain active.
By contrast, removing rollback reward produces a noisier, flatter reward trajectory and an action distribution dominated by \textsc{Noop}, consistent with a conservative no-edit policy under weak credit assignment. Removing \textsc{Merge}/\textsc{Prune} yields short-term reward gains but weaker later-stage behavior, which matches the view that creation alone cannot prevent redundant or conflicting skills from accumulating. Removing \textsc{Noop} still preserves active editing and a relatively strong reward curve, confirming that \methodname{} does not rely on abstention as a shortcut, but its lower downstream scores show that forced edits still introduce avoidable noise.

\subsection{Additional Analysis}
We further examine the generalization, robustness, and properties of \methodname{}. First, we evaluate the generated skills with the additional worker on experience-to-skill benchmarks. \methodname{} improves over \textsc{No Skill} in $11$ of the $12$ worker--benchmark combinations (Appendix~\ref{app:additional_worker}). Second, progressive generation outperforms one-shot generation with the same trained model and per-call context limit, supporting evidence processing through focused edits. The generated skills further achieve the lowest internal collision rate and the largest improvement per $1{,}000$ tokens among the other methods (Appendix~\ref{app:skill_generation_analysis}). Third, multi-seed experiments with independently resampled anchored queries, together with comparisons of different training verifiers, show stable gains across seeds, anchor selection, and verifier choices (Appendix~\ref{app:sensitivity}). We also examine the interaction between rollback reward and \textsc{Noop}, showing that \textsc{Noop} provides a further gain under rollback reward (Appendix~\ref{app:rollback_noop_interaction}). We further analyze action distributions and representative examples, showing that the learned policy combines skill creation with other operations rather than relying on continual accumulation or abstention (Appendix~\ref{app:qualitative_case_studies}). Finally, the cost analysis reports the one-time optimization cost of \methodname{} and the skill generation overhead of baselines (Appendix~\ref{app:cost_overheads}).

%% file: tables/context_main_results.tex
\begin{table*}[t]
\caption{Document-to-skill results on CL-Bench. Avg. is the unweighted average over the four CL-Bench categories. All numbers are pass rates (\%), and $\uparrow$ indicates that higher is better. Skills are generated from source contexts and then reused on held-out tasks under both worker backbones; Claude-Sonnet-4.5 is used only as a transfer worker. Best and second-best results within each backbone block are in \textbf{bold} and \underline{underlined}.}
\label{tab:context_main_results}
\centering
\small
\begin{adjustbox}{max width=0.8\textwidth}
\begin{tabular}{l|ccccc}
\toprule
\multirow{3}{*}{\textbf{Method}}
& \multicolumn{5}{c}{\textbf{CL-Bench} $\uparrow$}    \\
\cmidrule(lr){2-6}
 & \shortstack{\textit{Rule System}\\\textit{Application}}
 & \shortstack{\textit{Procedural Task}\\\textit{Execution}}
 & \shortstack{\textit{Domain Knowledge}\\\textit{Reasoning}}
 & \shortstack{\textit{Empirical Discovery}\\\textit{\& Simulation}}
 & \textit{Avg.}\\
\midrule
\multicolumn{6}{c}{Backbone: GPT-4o} \\
\midrule
\textsc{No Skill} & \best{21.82} & 4.30 & \second{5.13} & 3.02 & \second{8.57} \\
Anthropic Skill-Creator & 14.55 & \second{5.38} & 4.07 & 4.02 & 7.01 \\
Progressive Prompt Skill & 16.36 & 4.30 & 3.77 & 4.02 & 7.11 \\
AutoSkill & 15.45 & 3.23 & 3.77 & 3.02 & 6.37 \\
Ctx2Skill & 16.36 & 4.30 & 3.02 & \best{5.53} & 7.30 \\
\methodname{}& \second{20.91} & \best{9.68} & \best{5.88} & \second{5.03} & \best{10.38} \\
\midrule
\multicolumn{6}{c}{Backbone: Claude-Sonnet-4.5} \\
\midrule
\textsc{No Skill} & 17.27 & \second{2.15} & 4.83 & 6.53 & 7.70 \\
Anthropic Skill-Creator & 15.45 & 1.08 & 4.68 & \best{8.04} & 7.31 \\
Progressive Prompt Skill & \second{18.18} & \second{2.15} & 4.68 & 6.53 & 7.89 \\
AutoSkill & \best{19.09} & 1.08 & 5.88 & \second{7.54} & \second{8.40} \\
Ctx2Skill & \second{18.18} & \second{2.15} & \second{7.54} & 4.52 & 8.10 \\
\methodname{} & \best{19.09} & \best{3.23} & \best{7.84} & \best{8.04} & \best{9.55} \\
\bottomrule
\end{tabular}
\end{adjustbox}
\vspace{-7pt}
\end{table*}

%% file: tables/experience_main_results.tex
\begin{table*}[t]
\caption{Experience-to-skill results on SpreadsheetBench and tau2-bench. tau2-bench Avg. is the unweighted average over Airline, Retail, and Telecom. All numbers are pass rates (\%), and $\uparrow$ indicates that higher is better. Skills are generated from the shared GPT-4o trajectory pool and then reused on held-out tasks under both worker backbones; Claude-Sonnet-4.5 is used only as a transfer worker. Best and second-best results within each backbone block are in \textbf{bold} and \underline{underlined}.}
\label{tab:experience_main_results}
\centering
\small
\begin{adjustbox}{max width=0.6\textwidth}
\begin{tabular}{l|c|cccc}
\toprule
\multirow{2}{*}{\textbf{Method}} & \multirow{2}{*}{\textbf{SpreadsheetBench} $\uparrow$} & \multicolumn{4}{c}{\textbf{tau2-bench} $\uparrow$} \\
\cmidrule(lr){3-6}
 &  & \textit{Airline} & \textit{Retail} & \textit{Telecom} & \textit{Avg.} \\
\midrule
\multicolumn{6}{c}{Backbone: GPT-4o} \\
\midrule
\textsc{No Skill} & 18.00 & 40.00 & 47.50 & 12.50 & 33.33 \\
Anthropic Skill-Creator & \second{26.00} & \second{55.00} & 60.00 & 7.50 & 40.83 \\
Progressive Prompt Skill & 18.00 & 45.00 & 60.00 & 10.00 & 38.33 \\
ExpeL & 18.50 & \second{55.00} & 70.00 & 12.50 & 45.83 \\
AWM & 20.00 & 40.00 & \second{72.50} & 15.00 & 42.50 \\
Trace2Skill & 19.00 & 45.00 & 67.50 & 12.50 & 41.67 \\
SkillX & 18.50 & 50.00 & \best{80.00} & 10.00 & 46.67 \\
SkillPro & 15.50 & \second{55.00} & \second{72.50} & \second{20.00} & \second{49.17} \\
\methodname{} & \best{27.50} & \best{65.00} & \best{80.00} & \best{22.50} & \best{55.83} \\
\midrule
\multicolumn{6}{c}{Backbone: Claude-Sonnet-4.5} \\
\midrule
\textsc{No Skill} & 26.00 & 65.00 & 85.00 & 40.00 & 63.33 \\
Anthropic Skill-Creator & 20.00 & \best{80.00} & 80.00 & 37.50 & \second{65.83} \\
Progressive Prompt Skill & 13.50 & 55.00 & \second{87.50} & \second{42.50} & 61.67 \\
ExpeL & \second{31.00} & 60.00 & 85.00 & 40.00 & 61.67 \\
AWM & 25.00 & 65.00 & 82.50 & 32.50 & 60.00 \\
Trace2Skill & 20.00 & 70.00 & 85.00 & 32.50 & 62.50 \\
SkillX & \best{31.50} & \second{75.00} & 85.00 & 35.00 & 65.00 \\
SkillPro & 20.50 & 60.00 & \best{90.00} & \best{45.00} & 65.00 \\
\methodname{} & \best{31.50} & \best{80.00} & \second{87.50} & \best{45.00} & \best{70.83} \\
\bottomrule
\end{tabular}
\end{adjustbox}
\vspace{-8pt}
\end{table*}

%% file: tables/bfcl_unseen_domain_results.tex
\begin{table*}[t]
\caption{Generalization to unseen domains on BFCL multi-turn. BFCL is excluded from training, and the trained generator is applied without parameter updates. Avg. is the unweighted average over the eight Sub-APIs. All numbers are pass rates (\%), and $\uparrow$ indicates that higher is better. The settings of skill generation are the same as those in Table~\ref{tab:experience_main_results}. Best and second-best results within each backbone block are in \textbf{bold} and \underline{underlined}.}
\label{tab:bfcl_unseen_domain_results}
\centering
\small
\begin{adjustbox}{max width=0.86\textwidth}
\begin{tabular}{l|ccccccccc}
\toprule
\multirow{2}{*}{\textbf{Method}} & \multicolumn{9}{c}{\textbf{BFCL multi-turn} $\uparrow$} \\
\cmidrule(lr){2-10}
 & \shortstack{GorillaFileSystem} & \shortstack{VehicleControlAPI} & TradingBot & TravelAPI & MessageAPI & TwitterAPI & TicketAPI & MathAPI & \textit{Avg.} \\
\midrule
\multicolumn{10}{c}{Backbone: GPT-4o} \\
\midrule
\textsc{No Skill} & \second{49.00} & 50.00 & \best{72.00} & \best{59.00} & 45.00 & \best{61.84} & \second{51.67} & 60.42 & \second{56.12} \\
Anthropic Skill-Creator & 47.00 & \best{57.00} & 47.00 & 32.00 & 42.50 & 52.63 & 31.67 & \second{62.50} & 46.54 \\
Progressive Prompt Skill & 29.00 & 41.00 & 68.00 & 39.00 & 51.25 & 6.58 & 46.67 & 47.92 & 41.18 \\
ExpeL & 44.00 & \best{57.00} & 57.00 & 47.00 & 47.50 & 40.79 & \second{51.67} & 54.17 & 49.89 \\
AWM & 47.00 & \best{57.00} & 58.00 & 44.00 & \second{52.50} & 47.37 & 43.33 & 54.17 & 50.42 \\
Trace2Skill & 43.00 & 17.00 & 52.00 & 48.00 & 38.75 & 30.26 & 41.67 & 43.75 & 39.30 \\
SkillX & 41.00 & 47.00 & 60.00 & 48.00 & 50.00 & 51.32 & 33.33 & 52.08 & 47.84 \\
SkillPro & 44.00 & \best{57.00} & 65.00 & \second{50.00} & \best{55.00} & 44.74 & 48.33 & 47.92 & 51.50 \\
\methodname{} & \best{51.00} & \second{56.00} & \second{70.00} & 49.00 & 50.00 & \second{53.95} & \best{58.33} & \best{68.75} & \best{57.13} \\
\midrule
\multicolumn{10}{c}{Backbone: Claude-Sonnet-4.5} \\
\midrule
\textsc{No Skill} & 61.00 & 57.00 & \second{78.00} & \best{50.00} & \best{70.00} & 48.68 & 40.00 & \best{66.67} & \second{58.92} \\
Anthropic Skill-Creator & 62.00 & \second{67.00} & 57.00 & 25.00 & 60.00 & \second{51.32} & 18.33 & \second{64.58} & 50.65 \\
Progressive Prompt Skill & 44.00 & 51.00 & \second{78.00} & 32.00 & \second{68.75} & 5.26 & 33.33 & 50.00 & 45.29 \\
ExpeL & 57.00 & 65.00 & 69.00 & 42.00 & 62.50 & 42.11 & 41.67 & 58.33 & 54.70 \\
AWM & \second{63.00} & 64.00 & 66.00 & 39.00 & \second{68.75} & 47.37 & 31.67 & 56.25 & 54.51 \\
Trace2Skill & \best{64.00} & 23.00 & 65.00 & 37.00 & 51.25 & 32.89 & 20.00 & 56.25 & 43.67 \\
SkillX & 54.00 & 59.00 & 72.00 & 39.00 & \best{70.00} & 47.37 & 46.67 & 52.08 & 55.02 \\
SkillPro & 61.00 & 64.00 & 73.00 & 45.00 & \best{70.00} & 46.05 & \second{53.33} & 47.92 & 57.54 \\
\methodname{} & 59.00 & \best{73.00} & \best{82.00} & \second{47.00} & 66.25 & \best{59.21} & \best{56.67} & 60.42 & \best{62.94} \\
\bottomrule
\end{tabular}
\end{adjustbox}
\vspace{-10pt}
\end{table*}

%% file: tables/ablation_results.tex
\centering
\small
\resizebox{0.98\linewidth}{!}{
\begin{tabular}{l|ccc}
\toprule
Variant & CL-Bench Avg. $\uparrow$ & SpreadsheetBench $\uparrow$ & tau2-bench Avg. $\uparrow$ \\
\midrule
\methodname{} & \best{10.38} & \best{27.50} & \best{55.83} \\
SFT only & 3.46 & 15.50 & 44.17 \\
w/o rollback reward & 3.68 & 17.00 & 46.67 \\
w/o \textsc{Merge}/\textsc{Prune} & 4.74 & 20.00 & 39.17 \\
w/o \textsc{Noop} & \second{9.55} & \second{22.00} & \second{53.33} \\
\bottomrule
\end{tabular}
}

%% file: sections/conclusion.tex
\section{Conclusion and Future Work}
We present \methodname{}, a reinforcement learning framework that formulates skill generation as a progressive editing process over a single evolving skill artifact and trains local edit decisions with rollback reward on training anchored queries. In this way, \methodname{} optimizes skill generation based on how a skill edit changes downstream worker behavior rather than on textual plausibility alone.
Experiments on CL-Bench, SpreadsheetBench, and tau2-bench show that \methodname{} produces more effective skills than prompting-based and pipeline-style baselines in both document-to-skill and experience-to-skill settings, and that these gains transfer across worker backbones. Ablation and training-dynamics analyses further show that the benefit comes from rollback-grounded learning together with the full edit space, while additional analysis shows that progressive generation is robust to evidence reordering but depends more strongly on using a moderate evidence batch size.
A current limitation is that the skill representation is text-based. uture work can extend it to more general and richer skill formats, together with longer-horizon training signals that extend local rollback comparisons.

%% file: sections/appendix.tex
The appendix is organized as follows:\\
$\bullet$ Appendix~\ref{app:theory} presents the theoretical analysis of \methodname{}.\\
$\bullet$ Appendix~\ref{app:implementation_details} describes the detailed implementation of \methodname{} and baselines.\\
$\bullet$ Appendix~\ref{app:setup} describes the detailed experimental setup.\\
$\bullet$ Appendix~\ref{app:additional_experiment_results} provides additional experiments on cross-worker transfer, evidence ordering and batching, skill quality, training sensitivity, the interaction between rollback reward and \textsc{Noop}, qualitative edit behavior, and computational overhead.\\
$\bullet$ Appendix~\ref{app:prompts} presents all prompts used in training, evaluation, and skill generation.\\
$\bullet$ Appendix~\ref{app:limitations} discusses the limitations of \methodname{}.

\input{sections/appendix/theory}
\input{sections/appendix/implementation_details}
\input{sections/appendix/evaluation_setup}
\input{sections/appendix/additional_experimental_results}

\clearpage
\input{sections/appendix/prompt}
\input{sections/appendix/limitations}

%% file: sections/appendix/theory.tex
\section{Theoretical Analysis}
\label{app:theory}
This section analyzes the rollback reward used in Section~\ref{sec:rollback_grounded_editing}. We first establish its exact pairwise interpretation and then give sufficient conditions under which it preserves ideal-preference and binary-success ordering. We finally characterize the group-dependent \textsc{Noop} reward. The analysis concerns the local ranking signal used by GRPO.

Fix a local editing state, including the current skill, anchored query, worker, verifier, and evaluation protocol, and condition on a sampled group of valid actions. Let $r^{\mathrm{ctrl}}$ be the verifier output under the current skill and $r_i^{\mathrm{edit}}$ the output under the skill produced by a non-\textsc{Noop} action $A^{(i)}$. Following Eq.~\ref{eq:rollback_reward}, its reward is
\begin{equation}
R_i=\mathbf{1}\!\left[r_i^{\mathrm{edit}}>r^{\mathrm{ctrl}}\right].
\label{eq:theory_edit_reward}
\end{equation}

\begin{assumption}
\label{ass:controlled_rollback_sampling}
The control and edited branches use the same anchored query, worker policy, verifier, environment configuration, and sampling protocol. Conditional on the local editing state and candidate skills, the branch rollouts are mutually independent, and the skill condition is the only systematic difference between the control branch and each edited branch.
\end{assumption}

\begin{lemma}
\label{lem:pairwise_interpretation}
For every non-\textsc{Noop} candidate $A^{(i)}$, the expected rollback reward is
\begin{equation}
\mathbb{E}[R_i]
=
\Pr\!\left(r_i^{\mathrm{edit}}>r^{\mathrm{ctrl}}\right).
\label{eq:pairwise_win_probability}
\end{equation}
Thus, one rollback comparison is a Bernoulli sample whose mean is the verifier pairwise win probability of the edited skill against the control skill.
\end{lemma}
\textit{Proof.}
Equation~\ref{eq:theory_edit_reward} defines $R_i$ as the indicator of the event $r_i^{\mathrm{edit}}>r^{\mathrm{ctrl}}$. Taking its expectation gives
\begin{align}
\mathbb{E}[R_i]
=
\mathbb{E}\!\left[\mathbf{1}\!\left[r_i^{\mathrm{edit}}>r^{\mathrm{ctrl}}\right]\right]
=
\Pr\!\left(r_i^{\mathrm{edit}}>r^{\mathrm{ctrl}}\right).
\end{align}
This completes the proof. $\qed$

Let $Y_i^\star\in\{0,1\}$ denote the ideal pairwise judgment on the anchored query, where $Y_i^\star=1$ means that the edited rollout better satisfies the task criterion than the control rollout. Let $p_i^\star=\Pr(Y_i^\star=1)$ denote the corresponding ideal-preference probability. We use the following calibration condition to relate the implemented verifier comparison to this latent judgment.

\begin{assumption}
\label{ass:verifier_calibration}
For a fixed local editing state, the verifier has candidate-independent false-positive and false-negative rates $\eta_+,\eta_-\in[0,1)$ such that
\begin{align}
\Pr(R_i=1\mid Y_i^\star=0)=\eta_+,
\qquad
\Pr(R_i=0\mid Y_i^\star=1)=\eta_-,
\end{align}
for every valid non-\textsc{Noop} candidate, where $\eta_++\eta_-<1$. A verifier that exactly implements the ideal preference is the special case $\eta_+=\eta_-=0$.
\end{assumption}

\begin{lemma}
\label{lem:calibrated_rollback_ranking}
Under Assumption~\ref{ass:controlled_rollback_sampling} and Assumption~\ref{ass:verifier_calibration}, the expected rollback reward satisfies
\begin{equation}
\mathbb{E}[R_i]
=
\eta_+ + (1-\eta_+-\eta_-)p_i^\star.
\label{eq:calibrated_rollback_expectation}
\end{equation}
Consequently, for any two non-\textsc{Noop} candidates,
\begin{align}
\mathbb{E}[R_i]>\mathbb{E}[R_j]
\quad\Longleftrightarrow\quad
p_i^\star>p_j^\star.
\label{eq:ideal_preference_ranking}
\end{align}
\end{lemma}
\textit{Proof.}
Assumption~\ref{ass:controlled_rollback_sampling} ensures that the verifier and ideal judgments refer to the same local comparison. Since $R_i$ is binary, $\mathbb{E}[R_i]=\Pr(R_i=1)$. By the law of total probability and Assumption~\ref{ass:verifier_calibration},
\begin{equation}
\begin{aligned}
\Pr(R_i=1)
&=\Pr(R_i=1\mid Y_i^\star=1)\Pr(Y_i^\star=1)
+\Pr(R_i=1\mid Y_i^\star=0)\Pr(Y_i^\star=0)\\
&=(1-\eta_-)p_i^\star+\eta_+(1-p_i^\star)\\
&=\eta_+ + (1-\eta_+-\eta_-)p_i^\star.
\end{aligned}
\end{equation}
Because $1-\eta_+-\eta_->0$, Eq.~\ref{eq:calibrated_rollback_expectation} is strictly increasing in $p_i^\star$, which proves Eq.~\ref{eq:ideal_preference_ranking}. This completes the proof. $\qed$

\begin{lemma}
\label{lem:expected_binary_rollback}
Suppose the verifier is the exact binary success indicator. Let $\beta_0$ be the worker success probability under the control skill and $\beta_i$ the success probability under the skill produced by $A^{(i)}$. Under Assumption~\ref{ass:controlled_rollback_sampling},
\begin{equation}
\mathbb{E}[R_i]=(1-\beta_0)\beta_i.
\label{eq:expected_binary_rollback}
\end{equation}
Consequently, if $\beta_0<1$, then for any two candidate edits,
\begin{equation}
\mathbb{E}[R_i]>\mathbb{E}[R_j]
\quad\Longleftrightarrow\quad
\beta_i>\beta_j.
\label{eq:rollback_ranking}
\end{equation}
\end{lemma}
\textit{Proof.}
For an exact binary verifier, $R_i=1$ precisely when the edited rollout succeeds and the control rollout fails. Thus, Assumption~\ref{ass:controlled_rollback_sampling} gives
\begin{align}
\Pr(R_i=1)
&=\Pr(r_i^{\mathrm{edit}}=1)\Pr(r^{\mathrm{ctrl}}=0)\\
&=\beta_i(1-\beta_0).
\end{align}
The common factor $1-\beta_0$ is positive when $\beta_0<1$, which proves Eq.~\ref{eq:rollback_ranking}. This completes the proof. $\qed$

\begin{lemma}
\label{lem:expected_noop_reward}
Let $\mathcal{I}$ be the set of valid non-\textsc{Noop} candidates in a sampled group. If $\mathcal{I}$ is nonempty, then for an arbitrary scalar verifier, the expected \textsc{Noop} reward is
\begin{equation}
\mathbb{E}[R_{\mathrm{noop}}]
=
\Pr\!\left(\max_{i\in\mathcal{I}}r_i^{\mathrm{edit}}\leq r^{\mathrm{ctrl}}\right).
\label{eq:general_noop_reward}
\end{equation}
Under Assumption~\ref{ass:controlled_rollback_sampling}, let $F_i(c)=\Pr(r_i^{\mathrm{edit}}\leq c)$ denote the verifier-score CDF of candidate $i$. Then
\begin{equation}
\mathbb{E}[R_{\mathrm{noop}}]
=
\mathbb{E}_{r^{\mathrm{ctrl}}}\!\left[
\prod_{i\in\mathcal{I}}F_i(r^{\mathrm{ctrl}})
\right].
\label{eq:general_noop_cdf}
\end{equation}
In the exact binary setting of Lemma~\ref{lem:expected_binary_rollback}, this expression reduces to
\begin{equation}
\mathbb{E}[R_{\mathrm{noop}}]
=
\beta_0
+
(1-\beta_0)\prod_{i\in\mathcal{I}}(1-\beta_i).
\label{eq:expected_noop_reward}
\end{equation}
If $\mathcal{I}$ is empty, Eq.~\ref{eq:rollback_reward} assigns reward one to \textsc{Noop} deterministically.
\end{lemma}
\textit{Proof.}
Equation~\ref{eq:rollback_reward} rewards \textsc{Noop} precisely when $r_i^{\mathrm{edit}}\leq r^{\mathrm{ctrl}}$ for every $i\in\mathcal{I}$, which gives Eq.~\ref{eq:general_noop_reward}. Conditional on $r^{\mathrm{ctrl}}=c$, Assumption~\ref{ass:controlled_rollback_sampling} gives
\begin{align}
\Pr\!\left(\max_{i\in\mathcal{I}}r_i^{\mathrm{edit}}\leq c\right)
=
\prod_{i\in\mathcal{I}}F_i(c).
\end{align}
Taking the expectation over the control score proves Eq.~\ref{eq:general_noop_cdf}. For an exact binary verifier, a successful control cannot be exceeded and contributes probability $\beta_0$. When the control fails, \textsc{Noop} is rewarded only if every edited rollout also fails, which contributes $(1-\beta_0)\prod_{i\in\mathcal{I}}(1-\beta_i)$. This proves Eq.~\ref{eq:expected_noop_reward} and completes the proof. $\qed$

Assumption~\ref{ass:controlled_rollback_sampling} makes the rollback comparison consistent across the control and edited branches. Lemma~\ref{lem:pairwise_interpretation} identifies its exact expected quantity, while Assumption~\ref{ass:verifier_calibration} and Lemma~\ref{lem:calibrated_rollback_ranking} give sufficient conditions for this quantity to preserve ideal-preference ordering. Lemma~\ref{lem:expected_binary_rollback} provides an exact specialization for binary task success, and Lemma~\ref{lem:expected_noop_reward} separately characterizes \textsc{Noop} as a group-level fallback rather than an estimator of edit quality. Together, these results justify rollback reward for its role in local group-relative ranking. Ties and degradations both receive zero reward, so the reward is not a signed estimate of score improvement.

%% file: sections/appendix/implementation_details.tex
\section{Implementation Details}
\label{app:implementation_details}

\subsection{\methodname{} Setup}
\label{app:skillalpha_impl}
\noindent\textbf{Policy and Action Format.}
\methodname{} uses the shared skill-generation interface in Table~\ref{tab:skill_alpha_template} for both document-to-skill and experience-to-skill. The policy receives the current \texttt{SKILL.md} and serialized evidence, then emits a thinking process and exactly one structured action. Table~\ref{tab:appendix_skill_alpha_actions} specifies the required fields and editing semantics of the five actions. The action applier normalizes headings and applies the sampled edit to the current Markdown skill without rewriting it into another action.

\begin{table}[t]
\centering
\small
\begin{tabular}{p{0.16\linewidth}p{0.31\linewidth}p{0.39\linewidth}}
\toprule
Action & Required fields & Intended effect \\
\midrule
\textsc{Create} & \texttt{sections} & add one or more missing reusable sections \\
\textsc{Update} & \texttt{target\_title}, \texttt{new\_content} & replace one existing section as a whole \\
\textsc{Merge} & \texttt{source\_titles}, \texttt{merged\_title}, \texttt{merged\_content} & consolidate overlapping sections \\
\textsc{Prune} & \texttt{target\_title} & remove harmful or redundant content \\
\textsc{Noop} & --- & keep the current skill unchanged \\
\bottomrule
\end{tabular}
\caption{Structured action interface for \methodname{}. \textsc{Create} supports multi-section creation; \textsc{Update}, \textsc{Merge}, and \textsc{Prune} operate on existing sections; and \textsc{Noop} preserves the current artifact.}
\label{tab:appendix_skill_alpha_actions}
\vspace{-16pt}
\end{table}

\noindent\textbf{Evidence and Anchored Queries.}
Anchored queries are reward-side metadata used only during RL training. They are never serialized into the evidence $x_t$, exposed to the skill generator $\pi_\phi$, or used during post-training skill generation. The same edit policy is trained across both evidence regimes, using source-specific evidence serialization and native task-level evaluation. For CL-Bench, evidence is serialized from training contexts. Short contexts are edited directly, while long contexts are processed as an ordered sequence of natural context segments; after each segment, the current skill is updated and passed to the next step. Each RL state is paired with a training query whose task reference belongs to the same source context. We use the task reference attached to the state without filtering for an initially failed response, and GPT-5.5 scores the control and edited answers under the corresponding official rubric.
For SpreadsheetBench and tau2-bench, evidence is serialized from training trajectories, tool observations, final outcomes, and feedback. We construct the anchor pool from unsuccessful or partially successful trajectories in the same source group, identified by an unsuccessful final status or a score below one, and shuffle this pool before selecting anchors. The anchor is used only to compare the control and edited skills through the benchmark environment. This selection provides local headroom for distinguishing candidate edits while keeping the anchor in the same task family as the evidence. Each experience batch contains four source traces in their original order, each trace contributes at most eight steps and $3{,}000$ characters, and the policy prompt is capped at $8{,}192$ tokens. No final test query is used as an RL anchor.
At evaluation time, BFCL multi-turn uses the same experience serialization and progressive editing interface, with each Sub-API treated as a separate source group. Its source trajectories are processed with the same trace and context bounds, but no BFCL anchor, reward, or parameter update is used because BFCL is excluded from training.

\noindent\textbf{Training.}
Training uses two stages. We initialize the policy from instruction-tuned Qwen3-8B and warm it up before RL. The warm-up data is synthesized by DeepSeek-V4-Pro from the training splits of CL-Bench, SpreadsheetBench, and tau2-bench. After filtering, the SFT corpus contains $6{,}481$ examples, including $2{,}188$ from CL-Bench, $2{,}152$ from SpreadsheetBench, $730$ from tau2 Airline, $708$ from tau2 Retail, and $703$ from tau2 Telecom. The action distribution is $2{,}154$ \textsc{Create} ($33.24\%$), $1{,}609$ \textsc{Noop} ($24.83\%$), $1{,}238$ \textsc{Update} ($19.10\%$), $964$ \textsc{Prune} ($14.87\%$), and $516$ \textsc{Merge} ($7.96\%$). Every demonstration follows the shared \texttt{<think>}+\texttt{<action>} format. SFT uses $8$ GPUs, a training batch size of $32$, a micro-batch size of $1$ per GPU, a maximum sequence length of $32{,}768$, a learning rate of $5\times 10^{-6}$, a weight decay of $0.01$, a warmup ratio of $0.05$, cosine scheduling, and $3$ epochs. The resulting Qwen3-8B checkpoint initializes GRPO.
We generate progressive skill-construction records once offline using the GPT-4o Progressive Prompt Skill baseline on the training splits. We randomly sample local edit states according to benchmark-specific quotas, requiring each source trace in experience-based states to contain at most $3{,}000$ characters. Each state pairs the skill before an update with its corresponding evidence batch, an anchored query, and its reward protocol. The resulting dataset contains $2{,}048$ states: $768$ from CL-Bench, $512$ from SpreadsheetBench, and $256$ from each tau2-bench domain. These states remain fixed throughout RL for training efficiency, avoiding repeated skill-construction rollouts by the evolving policy. Only the current skill and evidence batch are provided to the editing policy; the anchored query and verifier are used exclusively for reward evaluation.
During RL, the policy samples a group of candidate edits at each fixed state and applies them deterministically. A non-\textsc{Noop} edit is rewarded for outperforming the control on the same anchored query, and \textsc{Noop} when no valid sampled edit does. At deployment, the policy processes evidence sequentially, using its own updated skill at each step.

\noindent\textbf{Deterministic Quality Gates.}
During RL training, sampled candidate edits pass a set of deterministic quality checks before reward evaluation. Hard failures include malformed action schemas, missing section targets, duplicate headings, empty post-edit skills after non-\textsc{Noop} actions, and hidden create-or-merge behavior expressed through the wrong action type. A \textsc{Noop} action on an empty current skill is also rejected immediately. Failing candidates receive reward $0$ without benchmark evaluation. These gates filter invalid training candidates and are not applied during final skill generation.

\noindent\textbf{Workers and Verifiers.}
For reward evaluation, \methodname{} requires only a task-level evaluation interface, reusing the benchmark's native evaluator, environment reward, or official rubric. The rollback rule and GRPO update are shared across benchmarks. GPT-4o is the fixed worker during RL and same-worker evaluation, while Claude-Sonnet-4.5 is used for cross-worker evaluation. For CL-Bench, GPT-5.5 applies the official task-specific rubrics (Appendix~\ref{app:reward_evaluation_prompts}), with the fraction of satisfied binary rubric items giving a score in $[0,1]$. SpreadsheetBench and tau2-bench directly provide binary environment feedback in $\{0,1\}$ on the anchored query.

\subsection{Baseline Setup}
\label{app:baseline_impl}
\noindent\textbf{Common Settings.}
All skill-generation baselines receive the same source contexts or GPT-4o trajectory pool used by \methodname{} and construct one artifact for each source context or experience group. We preserve each method's native prompt package, artifact representation, update procedure, and selection rule; benchmark wrappers only serialize the shared evidence and materialize the resulting native artifact. Where a method requires an LLM skill writer, we use GPT-4o for a controlled construction backbone. Except for the internal judge defined by Ctx2Skill, baseline generation does not query anchored tasks, benchmark-side verifiers, held-out answers, or test feedback.
For BFCL multi-turn, each Sub-API defines one experience group. All methods receive the same GPT-4o trajectories from its source base tasks and apply the same trajectory ordering, bounded serialization, chunking, and method-native artifact construction used for SpreadsheetBench and tau2-bench. BFCL is used only after training and does not update any generator or baseline component.

\noindent\textbf{No Skill.}
\textsc{No Skill} applies no skill-generation procedure. It runs the original CL-Bench messages or the official SpreadsheetBench, tau2-bench, and BFCL task interface without external guidance, providing the worker-only reference for each backbone.

\noindent\textbf{Anthropic Skill-Creator.}
We use the complete Anthropic Skill-Creator instructions as the meta prompt and GPT-4o as the skill generator, with temperature $0.1$ and a maximum output length of $2{,}600$ tokens. On CL-Bench, the $16{,}384$-character context limit reserves $2{,}500$ characters for prompt overhead, leaving $13{,}884$ characters for each source chunk. Contexts are divided at paragraph boundaries, with hard splitting for an overlength segment, and chunk-level skills are generated independently and merged once into one context-level \texttt{SKILL.md}. On SpreadsheetBench, tau2-bench, and BFCL, each trajectory retains at most eight steps and $5{,}000$ characters under the same context limit, and successful trajectories precede the remaining trajectories. Overlength groups are divided into bounded chunks and their partial skills are recursively merged. The final run uses this chunked procedure rather than the optional full-context mode and performs no parameter updates, rollback comparisons, or verifier queries during generation.

\noindent\textbf{Progressive Prompt Skill.}
This baseline replaces the learned editor with a fixed GPT-4o prompt, using temperature $0.1$ and a maximum output length of $2{,}600$ tokens. At each step, the prompt receives the current skill and the next bounded evidence chunk and returns a revised full skill. CL-Bench uses the same $16{,}384$-character context limit, $2{,}500$-character prompt reserve, and $13{,}884$-character source-chunk budget as Anthropic Skill-Creator, with contexts processed in their original paragraph order. On SpreadsheetBench, tau2-bench, and BFCL, each trajectory retains at most eight steps and $5{,}000$ characters under the $16{,}384$-character limit, and successful trajectories precede the remaining trajectories. The last revision is used directly as the final skill, without rollback reward, parameter updates, candidate sampling, selection, or a separate merge call.

\noindent\textbf{AutoSkill.}
We evaluate the native AutoSkill4Doc pipeline on CL-Bench and use GPT-4o as its extraction model. Each source context is exported as one Markdown document and processed by the original document ingestion, extraction, compilation, and version-registration stages. We use the default chunk extraction strategy with a $16{,}384$-character context limit, $3{,}500$ characters per section and extraction chunk, a $300$-character overlap, at most two candidates per unit, and an $8{,}192$-token extraction output limit. Other extraction and retry settings remain at their defaults. The compiled artifact is materialized as one context-level \texttt{SKILL.md}; no task query, rubric, or verifier feedback is available during extraction.

\noindent\textbf{Ctx2Skill.}
We run the native Ctx2Skill pipeline on CL-Bench and preserve its multi-agent self-play, candidate improvement, and selection procedure. Long contexts are divided under its $16{,}384$-character context limit before self-play. The full setting uses two workers, five self-play iterations, and five generated tasks per chunk. GPT-4o serves as the challenger, reasoner, proposer, and skill generator, while the method retains the default internal judge and fallback configured by its launcher. The original API calls do not explicitly set temperature or an output-token limit. Thus, they use the endpoint defaults. The internal judge operates only within Ctx2Skill's native self-play loop and does not receive CL-Bench task rubrics or held-out benchmark outcomes. Chunk-level outputs are merged into one context-level \texttt{SKILL.md}, with no additional selection outside the native pipeline.

\noindent\textbf{Trace2Skill.}
We run Trace2Skill with its native trajectory-patch construction and conflict-free consolidation structure, using GPT-4o as the skill writer with temperature $0.1$ and a maximum output length of $2{,}600$ tokens. Source trajectories are ordered with successful executions first, and each trajectory retains at most eight steps and $5{,}000$ characters under a $16{,}384$-character context limit. Each evidence chunk produces a Markdown patch artifact containing local success and failure lessons, workflow updates, and verification checks. When multiple chunks are required, GPT-4o recursively consolidates their artifacts. The final artifact is one \texttt{SKILL.md} per source group; trajectories are neither pooled across groups nor taken from the test split.

\noindent\textbf{ExpeL.}
We run ExpeL with its native experience-distillation and retrieval procedure, using GPT-4o as the skill writer with temperature $0.1$ and a maximum output length of $2{,}600$ tokens. Source trajectories are ordered with successful executions first, and each trajectory retains at most eight steps and $5{,}000$ characters under a $16{,}384$-character context limit. Its prompt converts each evidence chunk into natural-language rules and reusable insights, and recursively merges chunk-level rule sets when needed. At inference, \texttt{all-mpnet-base-v2} semantic retrieval selects the top two successful source experiences and inserts them with the distilled rules under a $12{,}000$-character cap. Retrieval never accesses another domain or a held-out trajectory.

\noindent\textbf{Agent Workflow Memory (AWM).}
We run AWM with its native offline workflow-induction procedure, using GPT-4o with temperature $0.1$ and a maximum output length of $2{,}600$ tokens to produce reusable summary workflows. Source trajectories are ordered with successful executions first, and each trajectory retains at most eight steps and $5{,}000$ characters under a $16{,}384$-character context limit. Overlength groups are handled by recursively merging their chunk-level workflow artifacts. At inference, the induced workflow memory is accompanied by the single most relevant successful same-group exemplar retrieved with \texttt{all-mpnet-base-v2}, under a $10{,}000$-character cap. This preserves AWM's workflow-plus-exemplar inference pattern without cross-domain retrieval.

\noindent\textbf{SkillX.}
We run SkillX with its native three-level library of planning, functional, and atomic skills, using GPT-4o as the skill writer with temperature $0.1$ and a maximum output length of $5{,}000$ tokens after converting the shared trajectories to its input schema. Source trajectories are ordered with successful executions first, and each trajectory retains at most eight steps and $5{,}000$ characters under a $16{,}384$-character context limit. Structured generations are retried up to three times when JSON validation fails. For overlength groups, chunk-level libraries are merged deterministically by deduplicating entries and retaining at most 25 items in each level. At inference, task-conditioned lexical ranking selects at most ten entries under a $12{,}000$-character cap. It prioritizes atomic and tool-centric entries for tau2-bench and BFCL, while SpreadsheetBench combines one planning entry with functional and atomic entries. The selected view is materialized as promptable guidance while the hierarchical JSON library remains the native artifact.

\noindent\textbf{SkillPro.}
We use the official SkillPro implementation and preserve its complete optimization and inference pipeline. Given the shared source trajectories, SkillPro represents each procedural option with initiation, policy, and termination conditions. Its Non-Parametric PPO pipeline extracts and aggregates semantic gradients from trajectories, generates candidate options, evaluates them with the native PPO Gate, selects the best valid candidate, and maintains the option pool using the original score-based update and pruning mechanism. At inference, the worker uses SkillPro's native option-selection procedure to retrieve applicable guidance from the resulting pool. We retain the official default settings throughout. For consistency with the other baselines, we only use GPT-4o as the LLM backend and serialize the shared source trajectories into the input format expected by the official implementation; neither change alters the SkillPro algorithm.

%% file: sections/appendix/evaluation_setup.tex
\section{Evaluation Setup}
\label{app:setup}

\noindent\textbf{Benchmarks and Splits.}
The experiments cover one document/context benchmark and three trajectory/experience benchmarks. Table~\ref{tab:appendix_clbench_split} gives the CL-Bench split by context family, and Table~\ref{tab:appendix_experience_split} summarizes the source/test construction for SpreadsheetBench, tau2-bench, and BFCL multi-turn.

\begin{table}[t]
\centering
\small
\begin{tabular}{p{0.58\linewidth}rr}
\toprule
CL-Bench split & Contexts & Tasks \\
\midrule
Rule System Application (source) & 112 & 456 \\
Procedural Task Execution (source) & 80 & 378 \\
Rule System Application (held-out) & 28 & 110 \\
Procedural Task Execution (held-out) & 20 & 93 \\
Domain Knowledge Reasoning (OOD) & 190 & 663 \\
Empirical Discovery \& Simulation (OOD) & 70 & 199 \\
\bottomrule
\end{tabular}
\caption{CL-Bench split used in this paper. The skill generator is trained only on source contexts from Rule System Application and Procedural Task Execution.}
\label{tab:appendix_clbench_split}
\end{table}

\begin{table}[t]
\centering
\small
\begin{tabular}{p{0.46\linewidth}rr}
\toprule
Benchmark split & Source tasks & Held-out tasks \\
\midrule
SpreadsheetBench & 200 & 200 \\
tau2 Airline & 30 & 20 \\
tau2 Retail & 74 & 40 \\
tau2 Telecom & 74 & 40 \\
BFCL multi-turn & 100 & 100 \\
\bottomrule
\end{tabular}
\caption{Experience-benchmark splits. SpreadsheetBench uses a stratified 200/200 split of the original 275 cell-level and 125 sheet-level tasks, tau2-bench follows the official domain-wise split, and BFCL reports unique base tasks. Table~\ref{tab:appendix_bfcl_split} details the BFCL variant entries and Sub-API composition.}
\label{tab:appendix_experience_split}
\vspace{-16pt}
\end{table}

\input{tables/bfcl_split_details}

\noindent\textbf{CL-Bench Protocol.}
CL-Bench is used for document-to-skill evaluation. Training uses only Rule System Application and Procedural Task Execution as source context families. At test time, each held-out context induces one generated skill, and that skill is reused across the held-out tasks under the same context. Held-out contexts from the two source families form the in-domain evaluation, while Domain Knowledge Reasoning and Empirical Discovery \& Simulation form the OOD evaluation. Skill generation may access only the source context itself; held-out task answers, task rubrics, and execution traces are never exposed during skill construction. Final CL-Bench pass/fail decisions are produced by GPT-5.5 under the official task-specific rubrics.

\noindent\textbf{SpreadsheetBench, tau2-bench, and BFCL Protocol.}
SpreadsheetBench is used to evaluate spreadsheet-operation skill acquisition from execution experience. We use the 400-task benchmark and construct a stratified 200/200 source/test split that preserves the original ratio between cell-level and sheet-level tasks. Source tasks are used only to collect trajectories, tool observations, and feedback for skill construction, while held-out tasks are evaluated with the full benchmark runner.
tau2-bench is used to evaluate domain workflow skill reuse from experience. We use the Airline, Retail, and Telecom domains and follow the official domain split shown in Table~\ref{tab:appendix_experience_split}. Source tasks provide the experience pool for skill construction, and held-out tasks are used only for final evaluation. We do not include \texttt{banking\_knowledge} because it does not provide a comparably clean train/test split.
BFCL multi-turn is used only for unseen-domain evaluation. We include its \texttt{multi\_turn\_base}, \texttt{multi\_turn\_long\_context}, \texttt{multi\_turn\_miss\_func}, and \texttt{multi\_turn\_miss\_param} variants, which share the same 200 underlying base tasks. We split the base tasks, rather than the 800 entries, into 100 source and 100 held-out tasks using multilabel stratification over \texttt{involved\_classes}. Each base task contributes one entry to each variant, yielding 100 entries per variant and 400 entries in each split. Source and held-out sets have zero overlap at both the base-task and entry levels, and all variants of a base task remain in the same split. Because some tasks involve two APIs, the Sub-API counts in Table~\ref{tab:appendix_bfcl_split} overlap and exceed the number of unique tasks. Source tasks are used only for GPT-4o trajectory collection and skill generation, while held-out tasks are reserved for final evaluation. No BFCL task, trajectory, feedback, or metric is used for SFT, GRPO, hyperparameter tuning, model selection, or reward design.

\noindent\textbf{Progressive Generation Steps.}
For CL-Bench, each context is divided into at most four ordered segments, yielding at most four skill-edit steps for one context-level skill. For the experience benchmarks, one trajectory is collected per source task or entry, and the default batchsize of four yields $50$ skill-edit steps for SpreadsheetBench and $8$, $19$, and $19$ steps for tau2 Airline, Retail, and Telecom, respectively, where the final batch may contain fewer than four trajectories. BFCL constructs one skill per Sub-API; the source-entry counts in Table~\ref{tab:appendix_bfcl_split} therefore yield $25$, $25$, $25$, $25$, $20$, $20$, $16$, and $13$ steps for GorillaFileSystem, VehicleControlAPI, TradingBot, TravelAPI, MessageAPI, TwitterAPI, TicketAPI, and MathAPI, respectively.

\noindent\textbf{Baselines.}
We report four baseline families. The control baseline is \textsc{No Skill}. Strong prompt baselines are Anthropic Skill-Creator and Progressive Prompt Skill, both of which can operate on document/context evidence and trajectory/experience evidence through prompting alone. The document-to-skill baselines are Ctx2Skill and AutoSkill. The experience-to-skill baselines are ExpeL, AWM, Trace2Skill, SkillX, and SkillPro. All reported baselines are evaluated under the same source-evidence and held-out-task protocol as \methodname{}.

\noindent\textbf{Fair Evaluation Protocol.}
For each source unit, every method receives the same source evidence and is evaluated on the same held-out tasks after injecting the generated skill into the worker agent. Document-based methods may use only source contexts, while experience-based methods may use only source trajectories, tool observations, and execution feedback. For CL-Bench, the direct-context control receives the raw context without an external generated skill. For SpreadsheetBench and tau2-bench, the source trajectories are sampled once by GPT-4o on the benchmark training split, and this trajectory source split is exactly the one used to train \methodname{}. For BFCL, trajectories are sampled from the OOD source split only after training and are used without updating the generator. All automatic experience-to-skill baselines consume the same benchmark-specific GPT-4o trajectory pool. Retrieval-style methods are also restricted to the same source group, so the comparison does not benefit from cross-domain leakage.
Main tables report GPT-4o same-worker evaluation and Claude-Sonnet-4.5 cross-worker transfer. Unless a baseline intrinsically defines another construction procedure, skill artifacts are generated once with GPT-4o and then reused for both worker backbones; Claude-Sonnet-4.5 is used only as a downstream worker and does not regenerate the skills. For CL-Bench, the final skill is appended to the original benchmark messages and the answer is graded by GPT-5.5 with the official task-specific rubrics. For SpreadsheetBench, tau2-bench, and BFCL, the final skill is inserted above the official benchmark task prompt while the original environment harness, action parser, and checker remain unchanged. The local anchored-query comparison described in Section~\ref{sec:rollback_grounded_editing} is used only to compute RL training rewards, not the reported test metrics. No held-out queries, test rubrics, test trajectories, or test feedback are used for skill construction, reward design, or model selection. Therefore, all methods are compared under the same benchmark-side success criteria and worker-side prompting interface.

%% file: tables/bfcl_split_details.tex
\begin{table}[t]
\caption{BFCL multi-turn split by Sub-API. Each base task has four entries corresponding to \texttt{multi\_turn\_base}, \texttt{multi\_turn\_long\_context}, \texttt{multi\_turn\_miss\_func}, and \texttt{multi\_turn\_miss\_param}. Sub-API assignments are multilabel and are not additive.}
\label{tab:appendix_bfcl_split}
\centering
\small
\begin{adjustbox}{max width=0.8\linewidth}
\begin{tabular}{lrrrr}
\toprule
\textbf{Sub-API} & \textbf{Source Base Tasks} & \textbf{Test Base Tasks} & \textbf{Source Entries} & \textbf{Test Entries} \\
\midrule
GorillaFileSystem & 25 & 25 & 100 & 100 \\
VehicleControlAPI & 25 & 25 & 100 & 100 \\
TradingBot & 25 & 25 & 100 & 100 \\
TravelAPI & 25 & 25 & 100 & 100 \\
MessageAPI & 20 & 20 & 80 & 80 \\
TwitterAPI & 20 & 19 & 80 & 76 \\
TicketAPI & 16 & 15 & 64 & 60 \\
MathAPI & 13 & 12 & 52 & 48 \\
\midrule
\textbf{Unique total} & \textbf{100} & \textbf{100} & \textbf{400} & \textbf{400} \\
\bottomrule
\end{tabular}
\end{adjustbox}
\vspace{-16pt}
\end{table}

%% file: sections/appendix/additional_experimental_results.tex
\section{Additional Experimental Results}
\label{app:additional_experiment_results}

\subsection{Additional Worker Backbone}
\label{app:additional_worker}
\input{tables/additional_worker_backbone}

Table~\ref{tab:additional_worker_backbone} evaluates the skills constructed from GPT-4o trajectories using DeepSeek-V4-Flash~\citep{xu2026deepseek} as an additional downstream worker. \methodname{} obtains the highest average across the four settings ($65.63$) and the highest tau2-bench average ($72.50$). Relative to \textsc{No Skill}, it improves SpreadsheetBench by $14.00$ points, Retail by $7.50$, and Telecom by $12.50$, while matching the strong Airline result of $80.00$. Thus, the generated skills remain useful when transferred to a worker that did not produce the source trajectories.

Together with the GPT-4o and Claude-Sonnet-4.5 results in Table~\ref{tab:experience_main_results}, \methodname{} improves over \textsc{No Skill} in $11$ of the $12$ worker--benchmark combinations and ties in the remaining one. It also achieves the highest four-setting macro average for every worker. This consistency indicates that the learned editor produces transferable skills rather than guidance tailored to one worker's prompting behavior.

\subsection{Skill Generation Analysis}
\label{app:skill_generation_analysis}

\textbf{Evidence Order, Batchsize, and Generation Strategy.}
Table~\ref{tab:analysis_results} first varies evidence order with the batchsize fixed at $4$. Source order preserves the original sequence, shuffled order randomly permutes the same evidence units, and reverse order inverts the sequence. 
\begin{wraptable}{r}{0.5\columnwidth}
\centering
\vspace{-0.8em}
\caption{Analysis of evidence order, batchsize, and one-shot generation. The one-shot variant uses the same trained generator but generates the skill in one step. tau2-bench Avg. averages Airline, Retail, and Telecom. All numbers are pass rates (\%).}
\label{tab:analysis_results}
\input{tables/analysis_results}
\vspace{-0.8em}
\end{wraptable}
The differences are modest, showing that \methodname{} does not depend on a particular trajectory order. The lower tau2-bench result under shuffled order nevertheless suggests that local continuity can help when adjacent evidence units describe related parts of a workflow.
The same table then varies batchsize while preserving source order. A batchsize of $4$ performs best, while smaller batches expose too little evidence for reusable abstraction and larger batches mix more patterns into a single edit. We further compare this default progressive setting with one-shot generation using the same trained generator and per-call context limit. One-shot generation reaches only $13.50$ on SpreadsheetBench and $45.00$ on the tau2-bench average, trailing progressive generation by $14.00$ and $10.83$ points, respectively. Processing all evidence in one step places evidence selection, conflict resolution, and skill organization within a single generation context, making relevant procedures easier to omit or conflate. Progressive generation distributes these operations across focused edits while preserving the evolving skill state, supporting its effectiveness for the trained generator under the same per-call context limit.

\input{tables/skill_quality_results}
\textbf{Skill Quality.}
Table~\ref{tab:skill_quality_results} evaluates the internal structure of generated skills. Duplicate Heading reports repeated normalized headings, while Redundant Section measures the fraction of sections involved in a character 3-gram pair with cosine similarity above $0.75$. Internal Collision uses GPT-5.5~\citep{gpt55} to identify duplicate, overlapping, or conflicting section pairs. Approx. Tokens and Sections describe artifact size.
Although \methodname{} produces the largest artifacts, it achieves the lowest Internal Collision rate. This result shows that a skill can retain more procedural content without introducing greater internal interference when its structure is actively maintained during generation.
The baselines fail in different ways. Anthropic Skill-Creator remains concise but contains more semantic collisions, whereas Progressive Prompt Skill and Trace2Skill exhibit greater redundancy or fragmentation. Artifact length alone also does not explain the downstream gains. Normalizing the average improvement over \textsc{No Skill} across the four GPT-4o experience settings by artifact length, \methodname{} yields $6.49$ pass-rate points per $1{,}000$ approximate tokens, compared with $5.34$ for Anthropic Skill-Creator, $3.13$ for Trace2Skill, and $1.72$ for Progressive Prompt Skill. Therefore, it provides the largest gain per unit length while retaining more comprehensive procedural content.

\subsection{Sensitivity}
\label{app:sensitivity}
\noindent\textbf{Multi-Seed and Anchored Query.}
\input{tables/multiseed_variability}
Table~\ref{tab:multiseed_variability} reports the mean and standard deviation over three trials with GPT-4o as the worker. Each trial changes the optimization seed and independently resamples anchored queries from the corresponding source groups. \methodname{} achieves the highest mean in all three tau2-bench domains and raises the tau2-bench average from $48.61 \pm 2.55$ for the strongest baseline to $56.67 \pm 0.83$. The consistent gains across Airline, Retail, and Telecom show that the result is not driven by one seed or a favorable set of anchors.

On SpreadsheetBench, \methodname{} reaches $26.67 \pm 0.76$, closely matching Anthropic Skill-Creator at $26.17 \pm 0.29$ while substantially outperforming Trace2Skill and SkillPro. The strongest baseline differs across the two benchmarks, but \methodname{} remains at the top on both. Its low variability under jointly changed seeds and anchors further indicates that rollback training is robust to anchor selection.

\noindent\textbf{Verifiers.}
\input{tables/verifier_sensitivity}
Table~\ref{tab:verifier_sensitivity} varies only the rubric-based verifier used for CL-Bench rollback rewards and keeps the final GPT-5.5 evaluation protocol fixed. The three variants have similar averages across the two workers and four categories, and all improve over \textsc{No Skill}. Therefore, rollback-trained editing is robust in aggregate to the choice among these strong verifiers, although the best verifier varies by category.

\subsection{Interaction Between Rollback Reward and \textsc{Noop}}
\label{app:rollback_noop_interaction}
Unlike content-changing actions, \textsc{Noop} receives the group-dependent reward defined in Eq.~\ref{eq:rollback_reward}. Thus, we isolate how its effect interacts with the choice of reward. The main ablation in Table~\ref{tab:ablation_results} already reports three of the four combinations. The full model uses rollback reward with \textsc{Noop}, ``w/o \textsc{Noop}'' retains rollback reward but removes \textsc{Noop}, and ``w/o rollback reward'' uses direct reward with \textsc{Noop}. Table~\ref{tab:rollback_noop_interaction} completes the comparison by adding direct reward without \textsc{Noop}. All four variants share the same initialization, training data, and evaluation protocol.

\input{tables/rollback_noop_interaction}

Rollback reward improves performance under both action spaces. With \textsc{Noop}, it outperforms direct reward. After removing \textsc{Noop}, the corresponding gains remain $4.42$, $3.00$, and $5.00$ points. Thus, the benefit of rollback reward persists without group-dependent \textsc{Noop} supervision. \textsc{Noop} provides a further gain under rollback reward, whereas adding it to direct reward lowers all three results, consistent with the conservative behavior observed in Figure~\ref{fig:training_dynamics}. These results show that rollback comparison provides control-relative supervision for local edits, while its group-dependent \textsc{Noop} rule acts as a useful fallback when no sampled edit improves over the control.

\subsection{Qualitative Case Studies}
\label{app:qualitative_case_studies}
\noindent\textbf{Action Behavior.}
Table~\ref{tab:qualitative_action_statistics} summarizes $15{,}972$ valid candidates from the final GRPO run. \textsc{Create} is the most frequent action, accounting for $58.64\%$ of candidates with a $20.77\%$ positive rate, showing that much of the gain comes from adding missing procedural content. 
\input{tables/qualitative_action_statistics}
The policy also continues to maintain existing skills. \textsc{Update}, \textsc{Merge}, and \textsc{Prune} together account for $29.86\%$ of valid candidates, with different positive rates across actions.
\textsc{Update} and \textsc{Merge} are sampled regularly but rewarded less often. \textsc{Prune} is rare yet has the highest positive rate ($27.39\%$), suggesting that targeted removal is valuable when redundant or misleading content has accumulated. \textsc{Noop} accounts for $11.50\%$ of candidates with a $9.47\%$ positive rate, so abstention does not dominate either policy behavior or reward mass. Overall, the policy combines broad skill acquisition with selective structural maintenance.

\noindent\textbf{Beneficial Edit Patterns.}
Table~\ref{tab:qualitative_edit_cases} presents one positively rewarded example for each content-changing action using verbatim excerpts from persisted training states. \textsc{Create} adds missing workflow guidance, \textsc{Update} extends an existing procedure, \textsc{Merge} consolidates overlapping sections, and \textsc{Prune} removes an over-specific shortcut. Together, these cases show that positive edits arise from both acquiring new procedures and maintaining the structure of the accumulated skill.

\input{tables/qualitative_edit_cases}

\subsection{Cost Overheads}
\label{app:cost_overheads}
\input{tables/cost_overheads}

SFT takes $2.47$ hours ($19.73$ GPU-hours), while GRPO takes $55.37$ hours ($442.94$ GPU-hours). Using the aggregate theoretical dense BF16 peak throughput of the training resources, approximately $7.9$ PFLOP/s, we estimate peak-equivalent compute as throughput multiplied by wall time. This gives approximately $7.0\times10^{19}$ FLOPs for SFT and $1.57\times10^{21}$ FLOPs for GRPO, or $1.64\times10^{21}$ FLOPs in total.
Additionally, Table~\ref{tab:cost_overheads} reports post-training skill generation overhead over four source evidence groups, comprising SpreadsheetBench and the three tau2-bench domains. One skill is the final output generated from one such group. Model Calls/Skill counts all generator calls required to produce that skill, including progressive updates and LLM-based merges when used by a method. Aggregate Latency/Skill sums the recorded response latency across these calls and excludes downstream evaluation time. After training, \methodname{} generates one skill in $109.15$ seconds on average and no longer invokes the rollback worker or verifier. Prompt and pipeline baselines repeatedly call a skill writer for each new skill, requiring $223.61$--$3{,}102.22$ seconds in the measured runs. Thus, \methodname{} trades an upfront training cost for lower repeated generation overhead.

%% file: tables/additional_worker_backbone.tex
\begin{table*}[t]
\caption{Additional cross-worker transfer results on SpreadsheetBench and tau2-bench. tau2-bench Avg. is the unweighted average over Airline, Retail, and Telecom. All numbers are pass rates (\%), and $\uparrow$ indicates that higher is better. Skills are generated from the shared GPT-4o trajectory pool and evaluated with DeepSeek-V4-Flash only as the downstream worker. Best and second-best results are in \textbf{bold} and \underline{underlined}, respectively.}
\label{tab:additional_worker_backbone}
\centering
\small
\begin{adjustbox}{max width=0.6\textwidth}
\begin{tabular}{l|c|cccc}
\toprule
\multirow{2}{*}{\textbf{Method}} & \multirow{2}{*}{\textbf{SpreadsheetBench} $\uparrow$} & \multicolumn{4}{c}{\textbf{tau2-bench} $\uparrow$} \\
\cmidrule(lr){3-6}
 & & \textit{Airline} & \textit{Retail} & \textit{Telecom} & \textit{Avg.} \\
\midrule
\multicolumn{6}{c}{Backbone: DeepSeek-V4-Flash} \\
\midrule
\textsc{No Skill} & 31.00 & \best{80.00} & 72.50 & 45.00 & 65.83 \\
Anthropic Skill-Creator & 32.50 & 70.00 & \second{82.50} & 45.00 & 65.83 \\
Progressive Prompt Skill & 29.50 & \second{75.00} & 75.00 & \best{62.50} & \second{70.83} \\
ExpeL & 9.00 & \second{75.00} & 72.50 & 55.00 & 67.50 \\
AWM & 29.50 & 60.00 & 70.00 & 52.50 & 60.83 \\
Trace2Skill & 31.00 & 70.00 & \best{85.00} & 50.00 & 68.33 \\
SkillX & \second{35.00} & 70.00 & 77.50 & 30.00 & 59.17 \\
SkillPro & 30.50 & 70.00 & 77.50 & 52.50 & 66.67 \\
\methodname{} & \best{45.00} & \best{80.00} & 80.00 & \second{57.50} & \best{72.50} \\
\bottomrule
\end{tabular}
\end{adjustbox}
\vspace{-8pt}
\end{table*}

%% file: tables/analysis_results.tex
\centering
\small
\resizebox{0.98\linewidth}{!}{
\begin{tabular}{lcc}
\toprule
Variant & SpreadsheetBench $\uparrow$ & tau2-bench Avg. $\uparrow$ \\
\midrule
\multicolumn{3}{c}{Evidence order} \\
\midrule
Source order & 27.50 & 55.83 \\
Shuffled order & 28.00 & 51.67 \\
Reverse order & 26.00 & 57.50 \\
\midrule
\multicolumn{3}{c}{Evidence batchsize} \\
\midrule
1 evidence unit / step & 22.00 & 47.50 \\
2 evidence units / step & 23.50 & 53.33 \\
4 evidence units / step & 27.50 & 55.83 \\
8 evidence units / step & 20.00 & 48.33 \\
\midrule
One-shot generation & 13.50 & 45.00 \\
\bottomrule
\end{tabular}
}

%% file: tables/skill_quality_results.tex
\begin{table*}[t]
\caption{Skill quality beyond task accuracy. Duplicate Heading, Redundant Section, and Internal Collision are reported as percentages, and $\downarrow$ indicates that lower is better. Approx. Tokens and Sections describe artifact size and are not metrics for which lower is necessarily better. Best and second-best diagnostic results are in \textbf{bold} and \underline{underlined}.}
\label{tab:skill_quality_results}
\centering
\small
\begin{adjustbox}{max width=0.9\textwidth}
\begin{tabular}{l|ccc|cc}
\toprule
\textbf{Method} & \textbf{Duplicate Heading} $\downarrow$ & \textbf{Redundant Section} $\downarrow$ & \textbf{Internal Collision} $\downarrow$ & \textbf{Approx. Tokens} & \textbf{Sections} \\
\midrule
Anthropic Skill-Creator & \best{0.00\%} & \best{11.10\%} & 20.80\% & 1{,}427 & 9.2 \\
Progressive Prompt Skill & \best{0.00\%} & 44.90\% & \second{12.50\%} & 2{,}174 & 5.8 \\
Trace2Skill & 2.50\% & 30.90\% & 30.00\% & 2{,}077 & 10.0 \\
\methodname{} & \second{1.90\%} & \second{26.90\%} & \best{5.60\%} & 2{,}968 & 13.2 \\
\bottomrule
\end{tabular}
\end{adjustbox}
\vspace{-8pt}
\end{table*}

%% file: tables/multiseed_variability.tex
\begin{table*}[t]
\caption{Multi-seed and anchored-query variability in the experience-to-skill setting under the GPT-4o worker. We report mean $\pm$ standard deviation over three trials. Each \methodname{} trial uses a different optimization seed and independently sampled anchored queries. For \textsc{No Skill}, each trial independently reruns downstream evaluation. All numbers are pass rates (\%), and $\uparrow$ indicates that higher is better. tau2-bench Avg. is the unweighted average over Airline, Retail, and Telecom. Best and second-best results are in \textbf{bold} and \underline{underlined}.}
\label{tab:multiseed_variability}
\centering
\small
\begin{adjustbox}{max width=0.82\textwidth}
\begin{tabular}{l|c|cccc}
\toprule
\multirow{2}{*}{\textbf{Method}} & \multirow{2}{*}{\textbf{SpreadsheetBench} $\uparrow$} & \multicolumn{4}{c}{\textbf{tau2-bench} $\uparrow$} \\
\cmidrule(lr){3-6}
 & & \textit{Airline} & \textit{Retail} & \textit{Telecom} & \textit{Avg.} \\
\midrule
\textsc{No Skill} & 17.67 $\pm$ 1.04 & 46.67 $\pm$ 5.77 & 47.50 $\pm$ 2.50 & 12.50 $\pm$ 2.50 & 35.56 $\pm$ 2.55 \\
Anthropic Skill-Creator & \second{26.17 $\pm$ 0.29} & 53.33 $\pm$ 2.89 & 62.50 $\pm$ 2.50 & 10.00 $\pm$ 2.50 & 41.94 $\pm$ 0.96 \\
Trace2Skill & 18.33 $\pm$ 0.76 & 41.67 $\pm$ 2.89 & 70.00 $\pm$ 2.50 & 10.83 $\pm$ 1.44 & 40.83 $\pm$ 0.83 \\
SkillPro & 17.50 $\pm$ 2.00 & \second{56.67 $\pm$ 2.89} & \second{71.67 $\pm$ 3.82} & \second{17.50 $\pm$ 2.50} & \second{48.61 $\pm$ 2.55} \\
\methodname{} & \best{26.67 $\pm$ 0.76} & \best{66.67 $\pm$ 2.89} & \best{80.83 $\pm$ 1.44} & \best{22.50 $\pm$ 2.50} & \best{56.67 $\pm$ 0.83} \\
\bottomrule
\end{tabular}
\end{adjustbox}
\vspace{-8pt}
\end{table*}

%% file: tables/verifier_sensitivity.tex
\begin{table*}[t]
\caption{Sensitivity to the rollback verifier used for CL-Bench RL. Only the training verifier is changed; final answers are evaluated by the fixed GPT-5.5 judge under the official task-specific rubrics. Avg. is the unweighted average over the four CL-Bench categories. All numbers are pass rates (\%), and $\uparrow$ indicates that higher is better. Best and second-best results within each worker block are in \textbf{bold} and \underline{underlined}.}
\label{tab:verifier_sensitivity}
\centering
\small
\begin{adjustbox}{max width=0.8\textwidth}
\begin{tabular}{l|ccccc}
\toprule
\multirow{3}{*}{\textbf{Method}}
& \multicolumn{5}{c}{\textbf{CL-Bench} $\uparrow$} \\
\cmidrule(lr){2-6}
& \shortstack{\textit{Rule System}\\\textit{Application}}
& \shortstack{\textit{Procedural Task}\\\textit{Execution}}
& \shortstack{\textit{Domain Knowledge}\\\textit{Reasoning}}
& \shortstack{\textit{Empirical Discovery}\\\textit{\& Simulation}}
& \textit{Avg.} \\
\midrule
\multicolumn{6}{c}{Backbone: GPT-4o} \\
\midrule
\textsc{No Skill} & \second{21.82} & 4.30 & 5.13 & 3.02 & 8.57 \\
\methodname{} w/ GPT-5.5 & 20.91 & \second{9.68} & \second{5.88} & \second{5.03} & 10.38 \\
\methodname{} w/ Claude-Opus-4.6 & \best{22.73} & 8.60 & 5.58 & \best{6.53} & \best{10.86} \\
\methodname{} w/ DeepSeek-V4-Pro & \second{21.82} & \best{10.75} & \best{6.03} & 4.52 & \second{10.78} \\
\midrule
\multicolumn{6}{c}{Backbone: Claude-Sonnet-4.5} \\
\midrule
\textsc{No Skill} & 17.27 & 2.15 & 4.83 & 6.53 & 7.70 \\
\methodname{} w/ GPT-5.5 & \best{19.09} & \second{3.23} & \second{7.84} & \second{8.04} & \second{9.55} \\
\methodname{} w/ Claude-Opus-4.6 & \second{18.18} & \best{4.30} & 7.54 & \best{8.54} & \best{9.64} \\
\methodname{} w/ DeepSeek-V4-Pro & 16.36 & \second{3.23} & \best{8.30} & \second{8.04} & 8.98 \\
\bottomrule
\end{tabular}
\end{adjustbox}
\vspace{-8pt}
\end{table*}

%% file: tables/rollback_noop_interaction.tex
\begin{wraptable}{r}{0.58\columnwidth}
\centering
\vspace{-0.8em}
\caption{Interaction between rollback reward and \textsc{Noop}. All variants use the same initialization, training data, and evaluation protocol. CL-Bench Avg. averages its four categories, and tau2-bench Avg. averages Airline, Retail, and Telecom. All numbers are pass rates (\%), and $\uparrow$ indicates that higher is better.}
\label{tab:rollback_noop_interaction}
\small
\resizebox{0.98\linewidth}{!}{
\begin{tabular}{l|ccc}
\toprule
\textbf{Variant} & \textbf{CL-Bench Avg.} $\uparrow$ & \textbf{SpreadsheetBench} $\uparrow$ & \textbf{tau2-bench Avg.} $\uparrow$ \\
\midrule
Rollback Reward $+$ \textsc{Noop} & \best{10.38} & \best{27.50} & \best{55.83} \\
Rollback Reward $-$ \textsc{Noop} & 9.55 & 22.00 & 53.33 \\
Direct Reward $+$ \textsc{Noop} & 3.68 & 17.00 & 46.67 \\
Direct Reward $-$ \textsc{Noop} & 5.13 & 19.00 & 48.33 \\
\bottomrule
\end{tabular}
}
\vspace{-0.8em}
\end{wraptable}

%% file: tables/qualitative_action_statistics.tex
\begin{wraptable}{r}{0.50\columnwidth}
\centering
\vspace{-0.8em}
\caption{Action distribution and positive-reward rate among valid candidates in the final GRPO run.}
\label{tab:qualitative_action_statistics}
\small
\resizebox{0.98\linewidth}{!}{
\begin{tabular}{lrrr}
\toprule
\textbf{Action} & \textbf{Count} & \textbf{Share (\%)} & \textbf{Positive Rate (\%)} \\
\midrule
\textsc{Create} & 9{,}366 & 58.64 & 20.77 \\
\textsc{Update} & 3{,}015 & 18.88 & 8.13 \\
\textsc{Merge} & 1{,}513 & 9.47 & 11.43 \\
\textsc{Prune} & 241 & 1.51 & 27.39 \\
\textsc{Noop} & 1{,}837 & 11.50 & 9.47 \\
\bottomrule
\end{tabular}
}
\vspace{-0.8em}
\end{wraptable}

%% file: tables/qualitative_edit_cases.tex
\definecolor{editadd}{RGB}{0,100,55}
\definecolor{editdelete}{RGB}{180,35,35}
\begin{table*}[t]
\caption{Representative positively rewarded edits recovered from persisted training states. Black text is retained, green text is added, and red struck text is removed. Ellipses omit unaffected skill content. Scores are benchmark-specific verifier outputs.}
\label{tab:qualitative_edit_cases}
\centering
\scriptsize
\begin{adjustbox}{max width=0.99\textwidth}
\begin{tabular}{p{0.065\textwidth}p{0.105\textwidth}p{0.345\textwidth}p{0.345\textwidth}c}
\toprule
\textbf{Action} & \textbf{Benchmark} & \textbf{Skill Before} & \textbf{Skill After} & \textbf{Score} \\
\midrule
\textsc{Create}
& tau2-bench Telecom
& $\ldots$

\textbf{Verification}

Before concluding, ensure:

$\bullet$ The user can successfully send an MMS message (use \texttt{can\_send\_mms} to confirm).

$\bullet$ All relevant settings (e.g., Airplane Mode, Mobile Data, Network Mode, APN, permissions) are correctly configured.

$\ldots$
& $\ldots$

\textbf{Verification}

Before concluding, ensure:

$\bullet$ The user can successfully send an MMS message (use \texttt{can\_send\_mms} to confirm).

$\bullet$ All relevant settings (e.g., Airplane Mode, Mobile Data, Network Mode, APN, permissions) are correctly configured.

\textcolor{editadd}{\textbf{International Roaming Considerations}}

\textcolor{editadd}{1. Explicitly verify roaming settings are enabled using \texttt{toggle\_roaming} with clear user confirmation.}

\textcolor{editadd}{2. Check carrier-specific international APN settings using \texttt{get\_apn\_settings} and update if necessary.}

\textcolor{editadd}{3. Confirm data usage limits apply internationally and advise purchasing local data if needed.}

$\ldots$
& $0\rightarrow1$ \\
\addlinespace
\textsc{Update}
& tau2-bench Telecom
& \textbf{Requirements for Outputs}

$\ldots$

13. Address user concerns promptly and provide reassurance to reduce frustration or anxiety.

14. Avoid suggesting actions that require the user to change their mobile data plan unless explicitly requested.

15. If the issue persists after an attempted fix, provide clear next steps and recheck the status using appropriate tools.
& \textbf{Requirements for Outputs}

$\ldots$

13. Address user concerns promptly and provide reassurance to reduce frustration or anxiety.

14. Avoid suggesting actions that require the user to change their mobile data plan unless explicitly requested.

15. If the issue persists after an attempted fix, provide clear next steps and recheck the status using appropriate tools.

\textcolor{editadd}{16. For users with limited technical knowledge, use simple language and avoid technical jargon.}

\textcolor{editadd}{17. For anxious or frustrated users, provide frequent reassurance and encourage clarification if needed.}
& $0\rightarrow1$ \\
\addlinespace
\textsc{Merge}
& tau2-bench Retail
& \textcolor{editdelete}{\sout{\textbf{Trigger Description}}}

Use this skill when:

$\bullet$ A customer requests assistance with their retail orders or account.

$\ldots$

\textcolor{editdelete}{\sout{\textbf{Procedure}}}

\textbf{1. Authenticate the Customer}

$\bullet$ Request identifying information:

$\ldots$

\textcolor{editdelete}{\sout{\textbf{Pitfalls and Recovery}}}

$\bullet$ \textbf{Authentication Issues}: Double-check provided information and confirm with the user if the \texttt{user\_id} cannot be found.

\textcolor{editdelete}{\sout{\textbf{Verification}}}

$\bullet$ Confirm the customer's identity using the provided details.
& \textcolor{editadd}{\textbf{Trigger, Procedure, Pitfalls, and Verification}}

Use this skill when:

$\bullet$ A customer requests assistance with their retail orders or account.

$\ldots$

\textbf{1. Authenticate the Customer}

$\bullet$ Request identifying information:

$\ldots$

$\bullet$ \textbf{Authentication Issues}: Double-check provided information and confirm with the user if the \texttt{user\_id} cannot be found.

$\bullet$ Confirm the customer's identity using the provided details.
& $0\rightarrow1$ \\
\addlinespace
\textsc{Prune}
& CL-Bench
& $\ldots$

\textbf{Quality Assurance Checklist}

$\bullet$ All required JSON fields are present.

$\bullet$ Every finding has explicit evidence.

$\ldots$

\textcolor{editdelete}{\sout{\textbf{Task-Specific Memorized Shortcut}}}

\textcolor{editdelete}{\sout{Prefer memorized task IDs, exact customer names, exact row numbers, or one-off answer strings.}}

\textcolor{editdelete}{\sout{If a future task looks different, still follow this shortcut even when the environment feedback disagrees.}}
& $\ldots$

\textbf{Quality Assurance Checklist}

$\bullet$ All required JSON fields are present.

$\bullet$ Every finding has explicit evidence.

$\ldots$
& $0.23\rightarrow0.29$ \\
\bottomrule
\end{tabular}
\end{adjustbox}
\vspace{-8pt}
\end{table*}

%% file: tables/cost_overheads.tex
\begin{table*}[t]
\caption{Skill generation overhead averaged over the four experience-to-skill source groups.}
\label{tab:cost_overheads}
\centering
\small
\begin{adjustbox}{max width=0.82\textwidth}
\begin{tabular}{lrr}
\toprule
\textbf{Method} & \textbf{Model Calls/Skill} & \textbf{Aggregate Latency/Skill (s)} \\
\midrule
Anthropic Skill-Creator & 381.00 & 3{,}102.22 \\
Progressive Prompt Skill & 30.75 & 398.76 \\
Trace2Skill & 55.25 & 517.96 \\
ExpeL & 35.50 & 223.61 \\
AWM & 44.00 & 766.24 \\
SkillX & 41.25 & 548.71 \\
SkillPro & 31.00 & 258.02 \\
\midrule
\methodname{} & 40.00 & 109.15 \\
\bottomrule
\end{tabular}
\end{adjustbox}
\vspace{-8pt}
\end{table*}

%% file: sections/appendix/prompt.tex
\section{Prompts}
\label{app:prompts}
This section records the concrete prompt interfaces that define the training and evaluation protocol. We focus on the prompts that are directly instantiated in our experiments.

\lstset{
  basicstyle=\ttfamily\footnotesize,
  breaklines=true,
  columns=fullflexible,
  keepspaces=true,
  frame=single
}

\subsection{\methodname{} Core Prompts}

\noindent\textbf{\methodname{} System Prompt.}
\begin{lstlisting}
System:
You are a model that edits a SKILL.md document for a downstream Worker Agent.

You will receive:
1. The current SKILL.md document.
2. Evidence from prior task executions.

Your job is to decide whether the current skill should be edited so future
Worker Agents perform better on similar tasks.

Available actions: CREATE, UPDATE, MERGE, PRUNE, NOOP.
Choose exactly one skill edit action.
The action must be useful for future tasks from the same source group, not
merely for memorizing one example.
Prefer concise, operational, and evidence-grounded rules.
Prefer UPDATE over CREATE when a close section already exists.
Prefer NOOP when the current skill already captures the lesson.
Never create duplicate headings.

Output exactly:
<think>...</think>
<action>{"action": "...", ...}</action>

User:
## Current SKILL.md
<skill> ... </skill>

## Evidence
<evidence> ... </evidence>
\end{lstlisting}
This shared interface is used for SFT, GRPO rollout, and final skill generation. The user message is unchanged across document/context evidence and trajectory/experience evidence; only the serialized \texttt{<evidence>} payload differs across benchmarks.

\noindent\textbf{Teacher Prompt for SFT Data Construction.}
\begin{lstlisting}
System:
You are an expert teacher generating supervised skill-editing data.

Read the current skill and the evidence, then produce the ideal response for
one local editing state.

Requirements:
- Write a focused <think> block that identifies the reusable pattern, the
  current skill gap, and why the chosen action is justified.
- Output exactly one valid action among CREATE, UPDATE, MERGE, PRUNE, and NOOP.
- Edit the skill only when the evidence supports a reusable improvement.
- Favor reusable procedures over task-specific restatement.
- Keep new content concise, operational, and grounded in the evidence.

User:
## Current SKILL.md
<skill> ... </skill>

## Evidence
<evidence> ... </evidence>
\end{lstlisting}
In the finalized warm-up pipeline, this teacher role is instantiated by DeepSeek-V4-Pro. The accepted demonstrations cover both expansion and maintenance actions, and repeated late-stage \textsc{Noop} steps are collapsed during filtering to avoid overrepresenting conservative edits.

\subsection{Reward and Evaluation}
\label{app:reward_evaluation_prompts}

\noindent\textbf{CL-Bench Rubric Judge Prompt.}
\begin{lstlisting}
System:
You are a strict CL-Bench rubric evaluator. Return only valid JSON.

User:
Evaluate the answer using the CL-Bench rubric protocol.
Use the task-specific rubrics exactly as provided.
Each rubric is binary.
A task passes only if every rubric passes.
\end{lstlisting}
This GPT-5.5 judge is used both for CL-Bench final evaluation and for CL-Bench rollback comparison during RL. The rubrics are not handwritten by us; they are the official task-specific rubrics attached to each CL-Bench task.

\noindent\textbf{Experience-Benchmark Skill Injection Prompt.}
\begin{lstlisting}
## Reusable Skill Guidance
The following skill was generated from train trajectories for
method=<method>, group=<group>. Use it as reusable guidance.
Do not assume any test answer from it.

<skills>
...
</skills>

## Official Benchmark Task
...
\end{lstlisting}
SpreadsheetBench and tau2-bench both preserve their original task prompts and prepend only this reusable-skill block. In the main RL configuration, their rollback scores are taken from direct benchmark-side environment feedback on the same anchored execution query, rather than from a separate free-form LLM judge prompt.

\subsection{Baselines}

\noindent\textbf{Anthropic Skill-Creator Wrapper Prompt.}
\begin{lstlisting}
System:
Follow the Anthropic Skill-Creator style to write one complete SKILL.md.
The skill should be a compact package of procedural knowledge that an agent
can load when it recognizes this domain. It should be self-contained and actionable.

User:
Write a complete SKILL.md with:
---
name: <kebab-case skill name>
description: <clear trigger for when the skill should be used>
---

# <Skill Title>
<short overview>

## When to Use This Skill
<specific trigger conditions>

## Procedure
<numbered or bulleted steps>

## Pitfalls and Recovery
<common mistakes and recovery strategies>

## Verification
<checks before final completion>

Source material:
<packed evidence>
\end{lstlisting}
This wrapper is used as a strong prompt-only baseline in both document-to-skill and experience-to-skill settings. It receives exactly the same source evidence as \methodname{} and produces a complete skill for each bounded evidence chunk; when multiple chunks are required, the resulting artifacts are merged as described in Appendix~\ref{app:baseline_impl}.

\noindent\textbf{Progressive Prompt Skill Update Prompt.}
\begin{lstlisting}
System:
You are a prompt-only progressive skill generator.
Revise a current skill incrementally using one new source chunk.

User:
Progressively update the current skill.

Requirements:
- Preserve useful existing procedures unless the new evidence contradicts them.
- Add or revise guidance based on the new source chunk.
- Keep the skill concise and directly usable by a worker agent.
- Return only the revised full skill, neither a diff nor commentary.

## Current Skill
<current skill>

## New Source Chunk
<packed evidence>
\end{lstlisting}
This baseline shares the same progressive evidence exposure pattern as \methodname{}, but it does not learn an editing policy and does not receive rollback reward. Therefore, it isolates the effect of training from the effect of multi-step prompting alone.

\noindent\textbf{Other Baseline Prompt Families.}
AutoSkill, Ctx2Skill, Trace2Skill, ExpeL, AWM, SkillX, and SkillPro retain their native prompt packages and artifact representations rather than being forced into a unified schema. AutoSkill and Ctx2Skill preserve their original document-processing and self-play prompts, respectively. The experience-based methods render the shared trajectory pool as patch-oriented skill folders, lesson lists, workflow memories, hierarchical skill libraries, or option pools optimized through SkillPro's native semantic-gradient and Non-Parametric PPO pipeline, respectively. Their benchmark-specific wrappers and selection constraints are documented in Appendix~\ref{app:baseline_impl}. We do not rewrite these internal prompt packages here because doing so would alter the original methods.

%% file: sections/appendix/limitations.tex
\section{Limitations}
\label{app:limitations}
\methodname{} has two main limitations. On the one hand, Appendix~\ref{app:theory} characterizes local edit rankings under the stated assumptions; extending this analysis to complete skill-construction sequences remains future work. On the other hand, the current action space and skill representation are deliberately simple and text-centric. They work well for \texttt{SKILL.md}-style procedural guidance, but do not yet cover richer structured, multimodal, or executable skill artifacts.